\documentclass[letterpaper]{article}
\usepackage[margin=1in,dvips]{geometry}
\usepackage{graphicx,psfrag,amssymb,amsthm,amsmath}
\usepackage{natbib}
\usepackage{url,color}

\newcommand{\F}{\ensuremath{\mathbf{F}}}

\newcommand{\J}{\ensuremath{\mathbf{J}}}

\newcommand{\W}{\ensuremath{\mathbf{W}}}
\newcommand{\X}{\ensuremath{\mathbf{X}}}
\newcommand{\Y}{\ensuremath{\mathbf{Y}}}
\newcommand{\Z}{\ensuremath{\mathbf{Z}}}

\newcommand{\f}{\ensuremath{\mathbf{f}}}
\newcommand{\g}{\ensuremath{\mathbf{g}}}

\newcommand{\w}{\ensuremath{\mathbf{w}}}
\newcommand{\x}{\ensuremath{\mathbf{x}}}
\newcommand{\y}{\ensuremath{\mathbf{y}}}
\newcommand{\z}{\ensuremath{\mathbf{z}}}
\newcommand{\0}{\ensuremath{\mathbf{0}}}

\newcommand{\blambda}{\ensuremath{\boldsymbol{\lambda}}}

\newcommand{\bTheta}{\ensuremath{\boldsymbol{\Theta}}}

\newcommand{\bbR}{\ensuremath{\mathbb{R}}}

\newcommand{\calI}{\ensuremath{\mathcal{I}}}

\newcommand{\calK}{\ensuremath{\mathcal{K}}}
\newcommand{\calL}{\ensuremath{\mathcal{L}}}
\newcommand{\calM}{\ensuremath{\mathcal{M}}}
\newcommand{\calN}{\ensuremath{\mathcal{N}}}

\newcommand{\abs}[1]{\left\lvert#1\right\rvert}
\newcommand{\norm}[1]{\left\lVert#1\right\rVert}

\newcommand{\caja}[4][1]{{%
    \renewcommand{\arraystretch}{#1}%
    \begin{tabular}[#2]{@{}#3@{}}%
      #4%
    \end{tabular}%
    }}

{%
\begin{list}{#1}{
\vspace{-\topsep}
\vspace{-\partopsep}
\setlength{\itemindent}{0cm}
\setlength{\rightmargin}{0cm}
\setlength{\listparindent}{0cm}
\settowidth{\labelwidth}{#1}
\setlength{\leftmargin}{\labelwidth}
\addtolength{\leftmargin}{\labelsep}
\setlength{\itemsep}{0cm}
}%
}%
{%
\end{list}
\vspace{-\topsep}
\vspace{-\partopsep}
}

{\begin{enumerate}%
}%
{\end{enumerate}}

\theoremstyle{plain}% default
\newtheorem{thm}{Theorem}[section]

\newtheorem*{lemma*}{Lemma}

\newtheorem*{prop*}{Proposition}

\theoremstyle{definition}

\newtheorem*{defn*}{Definition}

\newtheorem*{exmp*}{Example}

\newtheorem*{conj*}{Conjecture}

\theoremstyle{remark}

\newtheorem*{rmk*}{Remark}

\graphicspath{{grf/}}

\bibpunct[, ]{(}{)}{;}{a}{,}{,} % Should be the natbib default but it isn't!

\title{Distributed optimization of deeply nested systems}
\author{
  Miguel {\'A}.\ Carreira-Perpi{\~n}{\'a}n \hspace{5ex} Weiran Wang \\
  EECS, University of California, Merced \\
  {\small\url{http://eecs.ucmerced.edu}}
}
\date{December 24, 2012}

\begin{document}

\maketitle

\begin{abstract}
  
  In science and engineering, intelligent processing of complex signals such as images, sound or language is often performed by a parameterized hierarchy of nonlinear processing layers, sometimes biologically inspired. Hierarchical systems (or, more generally, nested systems) offer a way to generate complex mappings using simple stages. Each layer performs a different operation and achieves an ever more sophisticated representation of the input, as, for example, in an deep artificial neural network, an object recognition cascade in computer vision or a speech front-end processing. Joint estimation of the parameters of all the layers and selection of an optimal architecture is widely considered to be a difficult numerical nonconvex optimization problem, difficult to parallelize for execution in a distributed computation environment, and requiring significant human expert effort, which leads to suboptimal systems in practice. We describe a general mathematical strategy to learn the parameters and, to some extent, the architecture of nested systems, called the \emph{method of auxiliary coordinates (MAC)}. This replaces the original problem involving a deeply nested function with a constrained problem involving a different function in an augmented space without nesting. The constrained problem may be solved with penalty-based methods using alternating optimization over the parameters and the auxiliary coordinates. MAC has provable convergence, is easy to implement reusing existing algorithms for single layers, can be parallelized trivially and massively, applies even when parameter derivatives are not available or not desirable, and is competitive with state-of-the-art nonlinear optimizers even in the serial computation setting, often providing reasonable models within a few iterations.
  
\end{abstract}

\section{Introduction}

The continued increase in recent years in data availability and processing power has enabled the development and practical applicability of ever more powerful models in statistical machine learning, for example to recognize faces or speech, or to translate natural language \citep{Bishop06a}. However, physical limitations in serial computation suggest that scalable processing will require algorithms that can be massively parallelized, so they can profit from the thousands of inexpensive processors available in cloud computing. We focus on hierarchical processing architectures such as deep neural nets (fig.~\ref{f:deepnet}), which were originally inspired by biological systems such as the visual and auditory cortex in the mammalian brain \citep{RiesenPoggio99a,Serre_07a,GoldMorgan99a}, and which have been proven very successful at learning sophisticated tasks, such as recognizing faces or speech, when trained on data. A typical neural net defines a hierarchical, feedforward, parametric mapping from inputs to outputs. The parameters (\emph{weights}) are learned given a dataset by numerically minimizing an objective function. The outputs of the hidden units at each layer are obtained by transforming the previous layer's outputs by a linear operation with the layer's weights followed by a nonlinear elementwise mapping (e.g.\ sigmoid). Deep, nonlinear neural nets are universal approximators, that is, they can approximate any target mapping (from a wide class) to arbitrary accuracy given enough units \citep{Bishop06a}, and can have more representation power than shallow nets \citep{BengioLecun07a}. The hidden units may encode hierarchical, distributed features that are useful to deal with complex sensory data. For example, when trained on images, deep nets can learn low-level features such as edges and T-junctions and high-level features such as parts decompositions. Other examples of hierarchical processing systems are cascades for object recognition and scene understanding in computer vision \citep{Serre_07a,Ranzat_07b} or for phoneme classification in speech processing \citep{GoldMorgan99a,SaonChien12a}, wrapper approaches to classification or regression (e.g.\ based on dimensionality reduction; \citealp{WangCarreir12a}), or kinematic chains in robotics \citep{Craig04a}. These and other architectures share a fundamental design principle: \emph{mathematically, they construct a deeply nested mapping from inputs to outputs}.

The ideal performance of a nested system arises when all the parameters at all layers are jointly trained to minimize an objective function for the desired task, such as classification error (indeed, there is evidence that plasticity and learning probably occurs at all stages of the ventral stream of primate visual cortex; \citealp{RiesenPoggio99a,Serre_07a}). However, this is challenging because nesting (i.e., function composition) produces inherently nonconvex functions. Joint training is usually done with the backpropagation algorithm \citep{Rumelh_86c,Werbos74a}, which recursively computes the gradient with respect to each parameter using the chain rule. One can then simply update the parameters with a small step in the negative gradient direction as in gradient descent and stochastic gradient descent (SGD), or feed the gradient to a nonlinear optimization method that will compute a better search direction, possibly using second-order information \citep{OrrMueller98a}. This process is repeated until a convergence criterion is satisfied. Backprop in any of these variants suffers from the problem of vanishing gradients \citep{Roegnv94a,Erhan_09a}, where the gradients for lower layers are much smaller than those for higher layers, which leads to tiny steps, slowly zigzagging down a curved valley, and a very slow convergence. This problem worsens with the depth of the net and led researchers in the 1990s to give up in practice with nets beyond around two hidden layers (with special architectures such as convolutional nets \citep{Lecun_98a} being an exception) until recently, when improved initialization strategies \citep{HintonSalakh06a,Bengio_07a} and much faster computers---but not really any improvement in the optimization algorithms themselves---have renewed interest in deep architectures. Besides, backprop does not parallelize over layers (and, with nonconvex problems, is hard to parallelize over minibatches if using SGD), is only applicable if the mappings are differentiable with respect to the parameters, and needs careful tuning of learning rates. In summary, after decades of research in neural net optimization, simple backprop-based algorithms such as stochastic gradient descent remain the state-of-the-art, particularly when combined with good initialization strategies \citep{OrrMueller98a,HintonSalakh06a}. In addition, selecting the best architecture, for example the number of units in each layer of a deep net, or the number of filterbanks in a speech front-end processing, requires a combinatorial search. In practice, this is approximated with a manual trial-and-error procedure that is very costly in effort and expertise required, and leads to suboptimal solutions (where often the parameters of each layer are set irrespective of the rest of the cascade).

\begin{figure}[t]
  \begin{center}
    \psfrag{x}[l][lB]{\x}
    \psfrag{y}[l][lB]{\y}
    \psfrag{z1}[l][lB]{$\z_1$}
    \psfrag{z2}[l][lB]{$\z_2$}
    \psfrag{z3}[l][lB]{$\z_3$}
    \psfrag{W1}[][]{$\W_1$}
    \psfrag{W2}[][]{$\W_2$}
    \psfrag{W3}[][]{$\W_3$}
    \psfrag{W4}[][]{$\W_4$}
    \psfrag{s}[][lB]{$\sigma$}
    \includegraphics[width=0.22\textwidth]{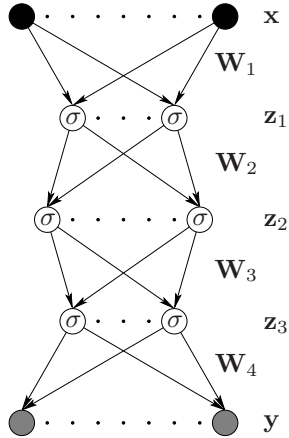}
    \caption{Net with $K=3$ hidden layers ($\z_k$: auxiliary coordinates, $\W_k$: weights).}
    \label{f:deepnet}
  \end{center}
\end{figure}

We describe a general optimization strategy for deeply nested functions that we call \emph{method of auxiliary coordinates (MAC)}, which partly alleviates the vanishing gradients problem, has embarrassing parallelization, and can reuse existing algorithms (possibly not gradient-based) that optimize single layers or individual units. Section~\ref{s:mac} describes MAC, section~\ref{s:related} describes related work, section~\ref{s:expts} gives experimental results that illustrate the different advantages of MAC, and the appendix gives formal theorem statements and proofs.

\section{The method of auxiliary coordinates (MAC)}
\label{s:mac}

\subsection{The nested objective function}

For definiteness, we describe the approach for a deep net such as that of fig.~\ref{f:deepnet}. Later sections will show other settings. Consider a regression problem of mapping inputs \x\ to outputs \y\ (both high-dimensional) with a deep net $\f(\x)$ given a dataset of $N$ pairs $(\x_n,\y_n)$. A typical objective function to learn a deep net with $K$ hidden layers has the form (to simplify notation, we ignore bias parameters):
\begin{equation}
  \label{e:nested}
  E_1(\W) = \frac{1}{2} \sum^N_{n=1}{\norm{\y_n - \f(\x_n;\W)}^2} \qquad \f(\x;\W) = \f_{K+1}(\dots \f_2(\f_1(\x;\W_1);\W_2)\dots;\W_{K+1})
\end{equation}
where each layer function has the form $\f_k(\x;\W_k) = \sigma(\W_k\x)$, i.e., a linear mapping followed by a squashing nonlinearity ($\sigma(t)$ applies a scalar function, such as the sigmoid $1/(1+e^{-t})$, elementwise to a vector argument, with output in $[0,1]$). Our method applies to loss functions other than squared error (e.g.\ cross-entropy for classification), with fully or sparsely connected layers each with a different number of hidden units, with weights shared across layers, and with regularization terms on the weights $\W_k$. The basic issue is the deep nesting of the mapping \f. The traditional way to minimize~\eqref{e:nested} is by computing the gradient over all weights of the net using backpropagation and feeding it to a nonlinear optimization method.

\subsection{The method of auxiliary coordinates (MAC)}

We introduce one auxiliary variable per data point and per hidden unit and define the following equality-constrained optimization problem:
\begin{equation}
  \label{e:mac}
  E(\W,\Z) = \frac{1}{2} \sum^N_{n=1}{\norm{\y_n - \f_{K+1}(\z_{K,n};\W_{K+1})}^2} \text{ s.t.\ }
  \renewcommand{\arraystretch}{0.5}
  \left\{
  \begin{array}{@{}l@{}}
    \z_{K,n} = \f_K(\z_{K-1,n};\W_K) \\ \dots \\ \z_{1,n} = \f_1(\x_n;\W_1)
  \end{array}
  \right\} n=1,\dots,N.
  \hspace{-3ex}
\end{equation}
Each $\z_{k,n}$ can be seen as the coordinates of $\x_n$ on an intermediate feature space, or as the hidden unit activations for $\x_n$. Intuitively, by eliminating \Z\ we see this is equivalent to the nested problem~\eqref{e:nested}; we can prove under very general assumptions that both problems have exactly the same minimizers (see appendix~\ref{s:proofs:equiv}). Problem~\eqref{e:mac} seems more complicated (more variables and constraints), but each of its terms (objective and constraints) involve only a small subset of parameters and no nested functions. Below we show this reduces the ill-conditioning caused by the nesting, and partially decouples many variables, affording an efficient and distributed optimization.

\subsection{MAC with quadratic-penalty (QP) optimization}

The problem~\eqref{e:mac} may be solved with a number of constrained optimization approaches. To illustrate the advantages of MAC in the simplest way, we use the quadratic-penalty (QP) method \citep{NocedalWright06a}. We optimize the following function over $(\W,\Z)$ for fixed $\mu>0$ and drive $\mu \rightarrow \infty$:
\begin{equation}
  \label{e:mac-quadpen}
  E_Q(\W,\Z;\mu) = \frac{1}{2} \sum^N_{n=1}{\norm{\y_n - \f_{K+1}(\z_{K,n};\W_{K+1})}^2} + \frac{\mu}{2} \sum^N_{n=1}{\sum^K_{k=1}{\norm{\z_{k,n} - \f_k(\z_{k-1,n};\W_k)}^2}}.
\end{equation}
This defines a continuous path $(\W^*(\mu),\Z^*(\mu))$ which, under some mild assumptions (see proof in appendix~\ref{s:proofs:QP}), converges to a minimum of the constrained problem~\eqref{e:mac}, and thus to a minimum of the original problem~\eqref{e:nested}. In practice, we follow this path loosely.

The QP objective function can be seen as breaking the functional dependences in the nested mapping \f\ and unfolding it over layers. Every squared term involves only a shallow mapping; all variables $(\W,\Z)$ are equally scaled, which improves the conditioning of the problem; and the derivatives required are simpler: we require no backpropagated gradients over \W, and sometimes no gradients over $\W_k$ at all. 

We now apply alternating optimization of the QP objective over \Z\ and \W:
\begin{description}
\item[\W-step] Minimizing over \W\ for fixed \Z\ results in a separate minimization over the weights of each hidden unit---each a single-layer, single-unit problem that can be solved with existing algorithms. Specifically, for the unit $(k,h)$, for $k=1,\dots,K+1$ (where we define $\z_{K+1,n} = \y_n$) and $h=1,\dots,H_k$ (assuming there are $H_k$ units in layer $k$), we have a nonlinear, least-squares regression of the form $\smash{\min_{\w_{kh}}{\sum^N_{n=1}{(z_{kh,n} - f_{kh}(\z_{k-1,n};\w_{kh}))}^2}}$, where $\w_{kh}$ is the weight vector ($h$th row of $\W_k$) that feeds into the $h$th output unit of layer $k$, and $z_{kh,n}$ the corresponding scalar target for point $\x_n$. 
\item[\Z-step] Minimizing over \Z\ for fixed \W\ separates over the coordinates $\z_n$ for each data point $n=1,\dots,N$ (omitting the subindex $n$ and weights):
\begin{equation}
  \label{e:Z-step}
  \min_{\z}{ \frac{1}{2} \norm{\y - \f_{K+1}(\z_K)}^2 + \frac{\mu}{2} \sum^K_{k=1}{\norm{\z_k - \f_k(\z_{k-1})}^2} }
\end{equation}
and can be solved using the derivatives w.r.t.\ \z\ of the single-layer functions $\f_1,\dots,\f_{K+1}$.
\end{description}
Thus, the \W-step results in many independent, single-layer single-unit problems that can be solved with existing algorithms, without extra programming cost. The \Z-step is new, however it always has the same form~\eqref{e:Z-step} of a ``generalized'' proximal operator \citep{Rockaf76b,CombetPesquet11a}. MAC reduces a complex, highly-coupled problem---training a deep net---to a sequence of simple, uncoupled problems (the \W-step) which are coordinated through the auxiliary variables (the \Z-step). For a large net with a large dataset, this affords an enormous potential for parallel, distributed computation. And, because each \W- or \Z-step operates over very large, decoupled blocks of variables, the decrease in the QP objective function is large in each iteration, unlike the tiny decreases achieved in the nested function. These large steps are effectively shortcuts through $(\W,\Z)$-space, instead of tiny steps along a curved valley in \W-space.

Rather than an algorithm, the method of auxiliary coordinates is a mathematical device to design optimization algorithms suited for any specific nested architecture, that are provably convergent, highly parallelizable and reuse existing algorithms for non-nested (or shallow) architectures. The key idea is the judicious elimination of subexpressions in a nested function via equality constraints. The architecture need not be strictly feedforward (e.g.\ recurrent nets). The designer need not introduce auxiliary coordinates at every layer: there is a spectrum between no auxiliary coordinates (full nesting), through hybrids that use some auxiliary coordinates and some semi-deep nets, to every single hidden unit having an auxiliary coordinate. An auxiliary coordinate may replace any subexpression of the nested function (e.g.\ the input to a hidden unit, rather than its output, leading to a quadratic \W-step). Other methods for constrained optimization may be used (e.g.\ the augmented Lagrangian rather than the quadratic-penalty method). Depending on the characteristics of the problem, the \W- and \Z-steps may be solved with any of a number of nonlinear optimization methods, from gradient descent to Newton's method, and using standard techniques such as warm starts, caching factorizations, inexact steps, stochastic updates using data minibatches, etc. In this respect, MAC is similar to other ``metaalgorithms'' such as expectation-maximization (EM) algorithms \citep{Dempst_77a} and alternating-direction method of multipliers \citep{Boyd_11a}, which have become ubiquitous in statistics, machine learning, optimization and other areas.

Fig.~\ref{f:usps} illustrates MAC learning for a sigmoidal deep autoencoder architecture, introducing auxiliary coordinates for each hidden unit at each layer (see section~\ref{s:expts:usps} for details). Classical backprop-based techniques such as stochastic gradient descent and conjugate gradients need many iterations to decrease the error, but each MAC/QP iteration achieves a large decrease, particularly at the beginning, so that it can reach a pretty good network pretty fast. While MAC/QP's serial performance is already remarkable, its parallel implementation achieves a linear speedup on the number of processors (fig.~\ref{f:parallel}).

\paragraph{Stopping criterion}

Exactly optimizing $E_Q(\W,\Z;\mu)$ for each $\mu$ follows the minima path strictly but is unnecessary, and one usually performs an inexact, faster optimization. Unlike in a general QP problem, in our case we have an accurate way to know when we should exit the optimization for a given $\mu$. Since our real goal is to minimize the nested error $E_1(\W)$, we exit when its value increases or decreases less than a set tolerance in relative terms. Further, as is common in neural net training, we use the validation error (i.e., $E_1(\W)$ measured on a validation set). This means we follow the path $(\W^*(\mu),\Z^*(\mu))$ not strictly but only inasmuch as we approach a nested minimum, and our approach can be seen as a sophisticated way of taking a descent step in $E_1(\W)$ but derived from $E_Q(\W,\Z;\mu)$. Using this stopping criterion maintains our theoretical convergence guarantees, because the path still ends in a minimum of $E_1$ and we drive $\mu\rightarrow\infty$.

\paragraph{The postprocessing step}

Once we have finished optimizing the MAC formulation with the QP method, we can apply a fast post-processing step that both reduces the objective function, achieves feasibility and eliminates the auxiliary coordinates. We simply satisfy the constraints by setting $\z_{kn} = \f_k(\z_{k-1,n};\W_k)$, $k=1,\dots,K$, $n=1,\dots,N$, and keep all the weights the same except for the last layer, where we set $\W_{K+1}$ by fitting $\f_{K+1}$ to the dataset $(\f_K(\dots(\f_1(\X))),\Y)$. One can prove the resulting weights reduce or leave unchanged the value of $E_1(\W)$.

\subsection{Jointly learning all the parameters in heterogeneous architectures}

Another important advantage of MAC is that it is easily applicable to heterogeneous architectures, where each layer may perform a particular type of processing for which a specialized training algorithm exists, possibly not based on derivatives over the weights (so that backprop is not applicable or not convenient). For example, a quantization layer of an object recognition cascade, or the nonlinear layer of a radial basis function (RBF) network, often use a $k$-means training to estimate the weights. Simply reusing this existing training algorithm as the \W-step for that layer allows MAC to learn jointly the parameters of the entire network with minimal programming effort, something that is not easy or not possible with other methods.

Fig.~\ref{f:coil} illustrates MAC learning for an autoencoder architecture where both the encoder and the decoder are RBF networks, introducing auxiliary coordinates only at the coding layer (see section~\ref{s:expts:coil} for details). In the \W-step, the basis functions of each RBF net are trained with $k$-means, and the weights in the remaining layers are trained by least-squares. As before, MAC/QP achieves a large error decrease in a few iterations.

\subsection{Model selection}

A final advantage of MAC is that it enables an efficient search not just over the parameter values of a given architecture, but (to some extent) over the architectures themselves. Traditional model selection usually involves obtaining optimal parameters (by running an already costly numerical optimization) for each possible architecture, and then evaluating each architecture based on a criterion such as cross-validation or a Bayesian Information Criterion (BIC), and picking the best \citep{Hastie_09a}. This discrete-continuous optimization involves training an exponential number of models, so in practice one settles with a suboptimal search (e.g.\ fixing by hand part of the architecture based on an expert's judgment, or selecting parts separately and then combining them). With MAC, model selection may be achieved ``on the fly'' by having the \W-step do model selection separately for each layer, and then letting the \Z-step coordinate the layers in the usual way. Specifically, consider a model selection criterion of the form $E_1(\W) + C(\W)$, where $E_1$ is the nested objective function~\eqref{e:nested} and $C(\W)$ is additive over the layers of the net:
\begin{equation}
  \label{e:model-selection-criterion}
  C(\W) = C_1(\W_1) + \dots + C_K(\W_K).
\end{equation}
This is satisfied by many criteria, such as BIC, AIC or minimum description length \citep{Hastie_09a}, in which $C(\W)$ is essentially proportional to the number of free parameters. While optimizing $E_1(\W) + C(\W)$ directly involves testing $M^K$ deep nets if we have $M$ choices for each layer, with MAC the \W-step separates over layers, and requires testing only $MK$ single-layer nets at each iteration. While these model selection tests are still costly, they may be run in parallel, and they need not be run at each iteration. That is, we may alternate between running multiple iterations that optimize \W\ for a given architecture, and running a model-selection iteration, and we still guarantee a monotonic decrease of $E_Q(\W) + C(\W)$. In practice, we observe that a near-optimal model is often found in early iterations. Thus, the ability of MAC to decouple optimizations reduces a search of an exponential number of complex problems to an iterated search of a linear number of simple problems.

Fig.~\ref{f:coil-model-selection} illustrates how to learn the architecture with MAC for the RBF autoencoder (see section~\ref{s:expts:coil-model-selection} for details) by trying $50$ different values for the number of basis functions in each of the encoder and decoder (a search space of $50^2 = 2\,500$ architectures). Because, early during the optimization, MAC/QP settles on an architecture that is quite smaller than the one used in fig.~\ref{f:coil}, the result is in fact achieved in even less time.

\section{Related work}
\label{s:related}

We believe we are the first to propose the MAC formulation in full generality for nested function learning as a provably equivalent, constrained problem that is to be optimized jointly in the space of parameters and auxiliary coordinates using quadratic-penalty, augmented Lagrangian or other methods. However, there exist several lines of work related to it, and MAC/QP can be seen as giving a principled setting that justifies previous heuristic but effective approaches, and opening the door for new, principled ways of training deep nets and other nested systems.

Updating the activations of hidden units separately from the weights of a neural net has been done in the past, from early work in neural nets \citep{Grossm_88a,SaadMarom90a,Krogh_90a,Rohwer90a} to recent work in learning sparse features \citep{OlshausField96a,OlshausField97a,Ranzat_07a,Kavukc_08a} and dimensionality reduction \citep{CarreirLu08a,CarreirLu10a,CarreirLu11a,WangCarreir12a}. Interest in using the activations of neural nets as independent variables goes back to the early days of neural nets, where learning good internal representations was as important as learning good weights \citep{Grossm_88a,SaadMarom90a,Krogh_90a,Rohwer90a}. In fact, backpropagation was presented as a method to construct good internal representations that represent important features of the task domain \citep{Rumelh_86d}. This necessarily requires dealing explicitly with the hidden activations. Thus, while several papers proposed objective functions of both the weights and the activations, these were not intended to solve the nested problem or to achieve distributed optimization, but to help learn good representations. These algorithms typically did not converge at all, or did not converge to a solution of the nested problem, and were developed for a single-hidden-layer net and tested in very small problems. More recent variations have similar problems \citep{Ma_97a,Castil_06a,Erdogm_05a}. Nearly all this early work has focused on the case of a single hidden layer, which is easy enough to train by standard methods, so that no great advantage is obtained, and it does not reveal the parallel processing aspects of the problem, which become truly important in the deep net case. When extracting features and using overcomplete dictionaries, sparsity is often encouraged, which sometimes requires an explicit penalty over the features, but this has only been considered for a single layer (the one that extracts the features) and again does not minimize the nested problem \citep{OlshausField96a,OlshausField97a,Ranzat_07a,Kavukc_08a}. Some work for a single hidden layer net mentions the possibility of recovering backpropagation in a limit \citep{Krogh_90a,Kavukc_08a}, but this is not used to construct an algorithm that converges to a nested problem optimum. Recent works in deep net learning, such as pretraining \citep{HintonSalakh06a} or greedy layerwise training \citep{Bengio_07a,Ngiam_11a}, do a single pass over the net from the input to the output layer, fixing the weights of each layer sequentially, but without optimizing a joint objective of all weights. While these heuristics can be used to achieve good initial weights, they do not converge to a minimum of the nested problem.

Auxiliary variables have been used before in statistics and machine learning, from early work in factor and homogeneity analysis \citep{Gifi90a}, to learn dimensionality reduction mappings given a dataset of high-dimensional points $\x_1,\dots,\x_N$. Here, one takes the latent coordinates $\z_n$ of each data point $\x_n$ as parameters to be estimated together with the reconstruction mapping \f\ that maps latent points to data space and minimize a least-squares error function $\smash{\sum^N_{n=1}{\norm{\x_n - \f(\z_n)}^2}}$, often by alternating over \f\ and \Z. Various nonlinear versions of this approach exist where \f\ is a spline \citep{LeblanTibshir94a}, single-layer neural net \citep{TanMavrov95a}, radial basis function net \citep{Smola_01a}, kernel regression \citep{Meinic_05a} or Gaussian process \citep{Lawren05a}. However, particularly with nonparametric functions, the error can be driven to zero by separating infinitely apart the \Z, and so these methods need ad-hoc terms on \Z\ to prevent this. The dimensionality reduction by unsupervised regression approach of \citet{CarreirLu08a,CarreirLu10a,CarreirLu11a} (generalized to supervised dimensionality reduction in \citealp{WangCarreir12a}) solves this by optimizing instead $\smash{\sum^N_{n=1}{\norm{\z_n - \F(\x_n)}^2 + \norm{\x_n - \f(\z_n)}^2}}$ jointly over \Z, \f\ and the projection mapping \F\ (both RBF networks). This can be seen as a truncated version of our quadratic-penalty approach, where $\mu$ is kept constant, and limited to a single-hidden-layer net. Therefore, the resulting estimate for the nested mapping $\f(\F(\x))$ is biased, as it does not minimize the nested error.

In summary, these works were typically concerned with single-hidden-layer architectures, and did not solve the nested problem~\eqref{e:nested}. Instead, their goal was to define a different problem (having a different solution): one where the designer has explicit control over the net's internal representations or features (e.g.\ to encourage sparsity or some other desired property). In MAC, the auxiliary coordinates are purely a mathematical construct to solve a well-defined, general nested optimization problem, with embarrassing parallelism suitable for distributed computation, and is not necessarily related to learning good hidden representations. Also, none of these works realize the possibility of using heterogeneous architectures with layer-specific algorithms, or of learning the architecture itself by minimizing a model selection criterion that separates in the \W-step.

Finally, the MAC formulation is similar in spirit to the alternating direction method of multipliers (ADMM) \citep{Boyd_11a,CombetPesquet11a} in that variables (the auxiliary coordinates) are introduced that decouple terms. However, ADMM splits an existing variable that appears in multiple terms of the objective function (which then decouple) rather than a functional nesting, for example $\min_{\x}{f(\x)+g(\x)}$ becomes $\min_{\x,\y}{f(\x)+g(\y)}$ s.t.\ $\x=\y$, or \x\ is split into non-negative and non-positive parts. In contrast, MAC introduces new variables to break the nesting. ADMM is known to be very simple, effective and parallelizable, and to be able to achieve a pretty good estimate pretty fast, thanks to the decoupling introduced and the ability to use existing optimizers for the subproblems that arise. MAC also has these characteristics with problems involving function nesting.

\section{Experiments}
\label{s:expts}

Section~\ref{s:expts:opt} describes how we implemented the \W- and \Z-steps and sections~\ref{s:expts:usps}, \ref{s:expts:coil} and~\ref{s:expts:coil-model-selection} show how MAC can learn a homogeneous architecture (deep sigmoidal autoencoder), a heterogeneous architecture (RBF autoencoder) and the architecture itself, respectively. In all cases, we show the speedup achieved with a parallel implementation of MAC as well.

\subsection{Optimization of the MAC-constrained problem using a quadratic penalty}
\label{s:expts:opt}

We apply alternating optimization of the QP objective~\eqref{e:mac-quadpen} over \Z\ and \W:
\begin{description}
\item[\W-step] Minimizing over \W\ for fixed \Z\ results in a separate nonlinear, least-squares regression of the form $\smash{\min_{\w_{kh}}{\sum^N_{n=1}{(z_{kh,n} - f_{kh}(\z_{k-1,n};\w_{kh}))}^2}}$ for $k=1,\dots,K+1$ (where we define $\z_{K+1,n} = \y_n$) and $h=1,\dots,H$, where $\w_{kh}$ is the weight vector ($h$th row of $\W_k$) that feeds into the $h$th output unit of layer $k$ (assuming there are $H$ such units), and $z_{kh,n}$ the corresponding scalar target for point $\x_n$. We solve each of these $KH$ problems with a Gauss-Newton approach \citep{NocedalWright06a}, which essentially approximates the Hessian by linearizing the function $f_{kh}$, solves a linear system of $H \times H$ (the number of units feeding into $z_{kh}$) to get a search direction, does a line search (we use backtracking with initial step size $1$), and iterates. In practice 1--2 iterations converge with high tolerance.
\item[\Z-step] Minimizing over \Z\ for fixed \W\ separates over each $\z_n$ for $n=1,\dots,N$. The problem is also a nonlinear least-squares fit, formally very similar to those of the \W-step, because \Z\ and \W\ enter the objective function in a nearly symmetric way through $\sigma(\W_k\z_{k-1})$, but with additional quadratic terms $\smash{\norm{\z_{k,n} - \dots}^2}$. We optimize it again with the Gauss-Newton method, which usually spends 1--2 iterations.
\end{description}
This optimization of the MAC-constrained problem, based on a quadratic-penalty method with Gauss-Newton steps, produces reasonable results and is simple to implement, but it is not intended to be particularly efficient. A more efficient optimization can be achieved by (1) using other methods for constrained optimization, such as the augmented Lagrangian method instead of the quadratic penalty method; and (2) by using more efficient \W- or \Z-steps, by combining standard techniques (inexact steps with warm starts, caching factorizations, stochastic updates using data minibatches, etc.\@) with unconstrained optimization methods such as L-BFGS, conjugate gradients, gradient descent, alternating optimization or others. Exploring this is a topic of future research.

\paragraph{Parallel implementation of MAC/QP}

Our parallel implementation of MAC/QP is extremely simple at present, yet it achieves large speedups (about $6 \times$ faster if using 12 processors), which are nearly linear as a function of the number of processors for all experiments, as shown in fig.~\ref{f:parallel}. Given that our code is in Matlab, we used the Matlab Parallel Processing Toolbox. The programming effort is insignificant: all we do is replace the ``for'' loop over weight vectors (in the \W-step) or over auxiliary coordinates (in the \Z-step) with a ``parfor'' loop. Matlab then sends each iteration of the loop to a different processor. We ran this in a shared-memory multiprocessor machine%
\footnote{An Aberdeen Stirling 148 computer having 4 physical CPUs (Intel Xeon CPU L7555@ 1.87GHz), each with 8 individual processing cores (thus a total of 32 actual processors), and a total RAM size of 64 GB.},
using up to 12 processors (a limit imposed by our Matlab license) and obtained the results reported in the paper. While simple, the Matlab Parallel Processing Toolbox is quite inefficient. Larger speedups would be achievable with other parallel computation models such as MPI in C, and using a distributed architecture (so that cache and other overheads are reduced).

\subsection{Homogeneous training: deep sigmoidal autoencoder}
\label{s:expts:usps}

We use a dataset of handwritten digit images to train a deep autoencoder architecture that maps the input image to a low-dimensional coding layer and then tries to reconstruct the image from it. We used MAC/QP introducing auxiliary coordinates for each hidden unit at each layer. Fig.~\ref{f:usps} shows the learning curves.

The USPS dataset \citep{Hull94a}, a commonly used machine learning benchmark, contains $16\times 16$ grayscale images of handwritten digits, i.e., 256D vectors with values in $[0,1]$. We use $N=5\,000$ images for training and $5\,000$ for validation, both randomly selected equally over all digits.

The autoencoder architecture is 256--300--100--20--100--300--256, for a total of over $200\,000$ weights, with all $K=5$ hidden layers being logistic sigmoid units and the output layer being linear. The initial weights are uniformly sampled from $[-1/\sqrt{f_k},1/\sqrt{f_k}]$ for layers $k=1,\dots,K$, respectively, where $f_k$ is the input dimension (fan-in) to each hidden layer \citep{OrrMueller98a}. When using initial random weights, a large (but easy) decrease in the error can be achieved simply by adjusting the biases of the output layer so the network output matches the mean of the target data, and all algorithms attain this in the first few iterations, giving the impression of a large error decrease followed by a very slow subsequent decrease. To focus the plots on the subsequent, more difficult error decreases, we always apply a single step of gradient descent to the random initial weights, which mostly adjusts the biases as described. The resulting weights are given as initial weights to each optimization method.

MAC/QP runs the \Z-step with 1 Gauss-Newton iteration and the \W-step with up to 3 Gauss-Newton iterations. We also use a small regularization on \W\ in the first few iterations, which we drop once $\mu > 10^4$, since we find that this tends to lead to a better local optimum. For each value of $\mu$, we optimize $E_Q(\W,\Z;\mu)$ in an inexact way as described in section~\ref{s:mac}, exiting when the value of the nested error $E_1(\W)$, evaluated on a validation set, increases or decreases less than a set tolerance in relative terms. We use a tolerance of $10^{-2}$ and increase $\mu$ rather aggressively, to $10 \mu$.

We show the learning curves for two classical backprop-based techniques, stochastic gradient descent and conjugate gradients. Stochastic gradient descent (SGD) has several parameters that should be carefully set by the user to ensure that convergence occurs, and that convergence is as fast as possible. We did a grid search for the minibatch size in $\{1, 10, 20, 50, 100, 200, 500, 1000, 5000\}$ and learning rate in $\{1, 10, 100, 1000\} \times 10^{-7}$. We found minibatches of 20 samples and a learning rate of $10^{-6}$ were best. We randomly permute the training set at the beginning of each epoch. For nonlinear conjugate gradients (CG), we used the Polak-Ribi{\`e}re version, widely regarded as the best \citep{NocedalWright06a}. We use Carl Rasmussen's implementation \url{minimize.m} (available at \url{http://learning.eng.cam.ac.uk/carl/code/minimize/minimize.m}), which uses a line search based on cubic interpolation that is more sophisticated than backtracking, and allows steps longer than 1. We found a minibatch of $N$ (i.e., batch mode) worked best. We used restarts every 100 steps. 

Figure~\ref{f:usps}(left) plots the mean squared training error for the nested objective function~\eqref{e:nested} vs run time for the different algorithms after the initialization. The validation error follows closely the training error and is not plotted. Markers are shown every iteration (MAC), every 100 iterations (CG), or every 20 epochs (SGD, one epoch is one pass over the training set). The change points of the quadratic penalty parameter $\mu$ are indicated using filled red circles at the beginning of each iteration when they happened. The learning curve of the parallelized version of MAC (using 12 processors) is also shown in blue. Fig.~\ref{f:usps}(right) shows some USPS images and their reconstruction at the end of training for each algorithm.

SGD and CG need many iterations to decrease the error, but each MAC/QP iteration achieves a large decrease, particularly at the beginning, so that it can reach a pretty good network pretty fast. While MAC/QP's serial performance is already remarkable, its parallel implementation achieves a linear speedup on the number of processors (fig.~\ref{f:parallel}).

\newlength{\MACPlengthA}
\setlength{\MACPlengthA}{0.08\linewidth}
\begin{figure}[t]
  \begin{center}
    \begin{tabular}[c]{@{}c@{\hspace{0.08\linewidth}}c@{}}
      \begin{tabular}[c]{@{}c@{}}
        \psfrag{mu0}{{\tiny $\mu=1$}}
        \psfrag{mu1}[][]{{\tiny $\mu=10^1$}}
        \psfrag{mu2}[][]{{\tiny $10^2$}}
        \psfrag{mu3}[][]{{\tiny $10^3$}}
        \psfrag{mu4}[][]{{\tiny $10^4$}}
        \psfrag{mu6}[][]{{\tiny $10^6$}}
        \psfrag{mu7}[][]{{\tiny $10^7$}}
        \psfrag{mu8}[][]{{\tiny $10^8$}}
        \psfrag{mse}[][]{objective function}
        \psfrag{time}[][]{runtime (hours)}
        \psfrag{MAC}{MAC}
        \psfrag{CG}{CG}
        \psfrag{SGD}{SGD}
        \psfrag{Parallel MAC}{Parallel MAC}
        \includegraphics[width=0.48\linewidth]{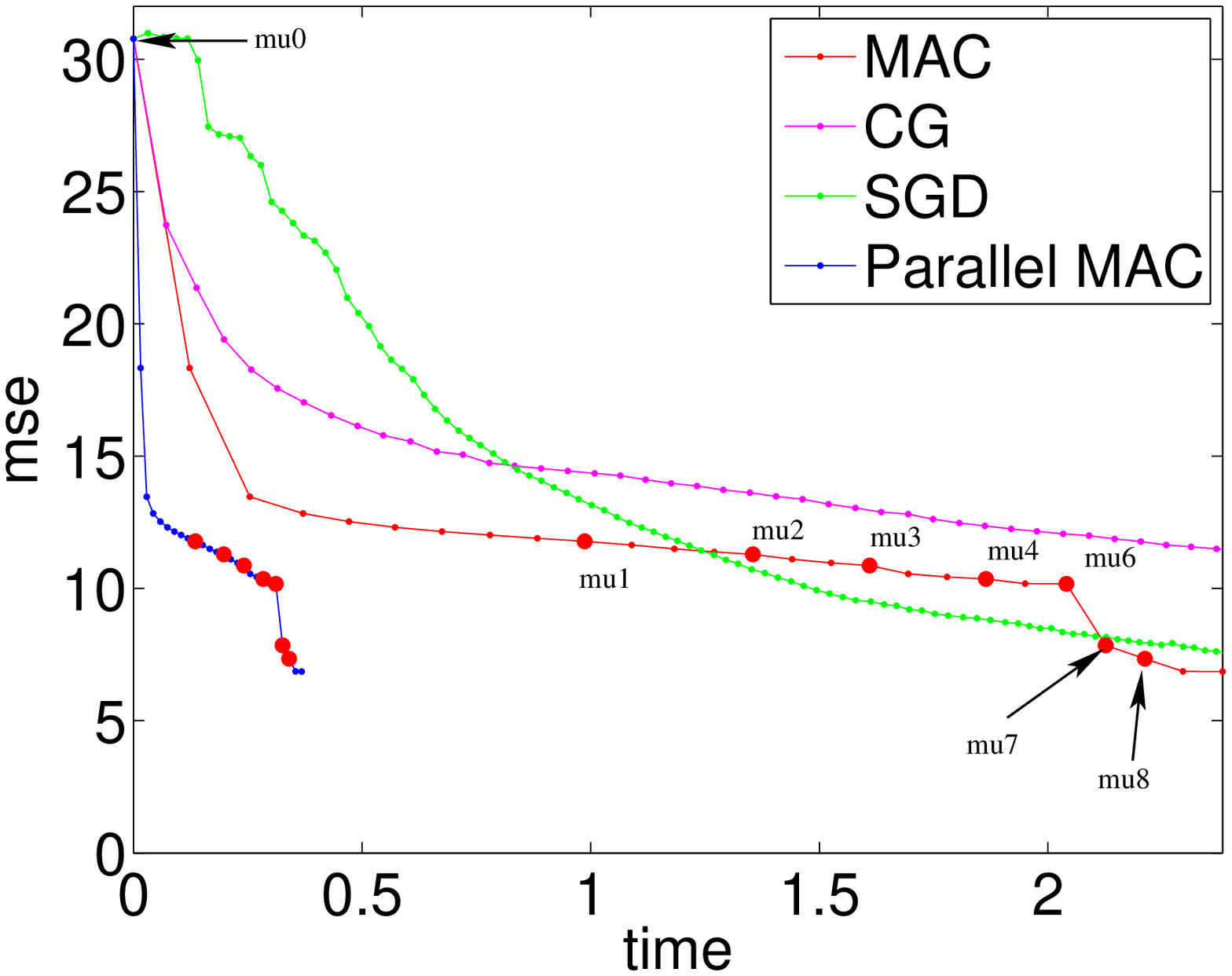}
      \end{tabular} &
      \begin{tabular}{@{}c@{\hspace{0\linewidth}}c@{\hspace{0\linewidth}}c@{\hspace{0\linewidth}}c@{\hspace{0\linewidth}}c@{\hspace{0\linewidth}}c@{}}
        \rotatebox{90}{\makebox[\MACPlengthA][c]{\caja[0.5]{c}{c}{Ground \\ truth}}} &
        \includegraphics[width=\MACPlengthA,bb=280 371 330 421,clip]{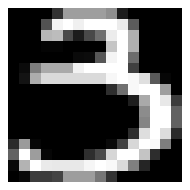} &
        \includegraphics[width=\MACPlengthA,bb=280 371 330 421,clip]{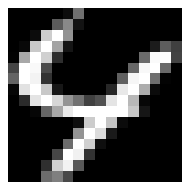} &  
        \includegraphics[width=\MACPlengthA,bb=280 371 330 421,clip]{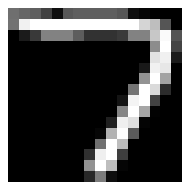} &
        \includegraphics[width=\MACPlengthA,bb=280 371 330 421,clip]{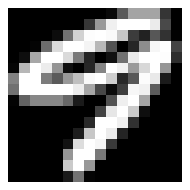} &
        \includegraphics[width=\MACPlengthA,bb=280 371 330 421,clip]{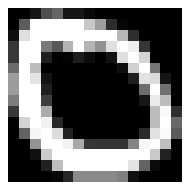} \\[-0.5ex]
        \rotatebox{90}{\makebox[\MACPlengthA][c]{MAC}} &
        \includegraphics[width=\MACPlengthA,bb=280 371 330 421,clip]{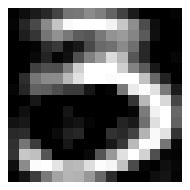} &
        \includegraphics[width=\MACPlengthA,bb=280 371 330 421,clip]{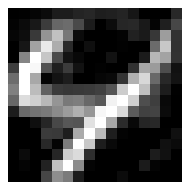} &  
        \includegraphics[width=\MACPlengthA,bb=280 371 330 421,clip]{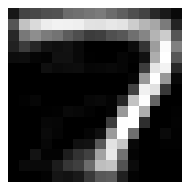} &
        \includegraphics[width=\MACPlengthA,bb=280 371 330 421,clip]{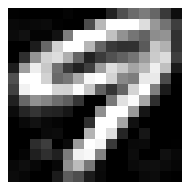} &
        \includegraphics[width=\MACPlengthA,bb=280 371 330 421,clip]{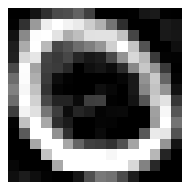} \\[-0.5ex]
        \rotatebox{90}{\makebox[\MACPlengthA][c]{CG}} &
        \includegraphics[width=\MACPlengthA,bb=280 371 330 421,clip]{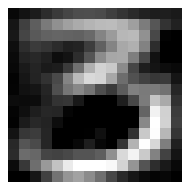} &
        \includegraphics[width=\MACPlengthA,bb=280 371 330 421,clip]{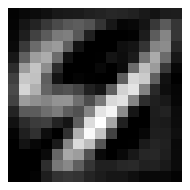} &  
        \includegraphics[width=\MACPlengthA,bb=280 371 330 421,clip]{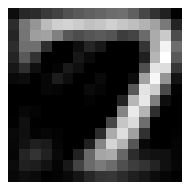} &
        \includegraphics[width=\MACPlengthA,bb=280 371 330 421,clip]{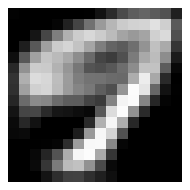} &
        \includegraphics[width=\MACPlengthA,bb=280 371 330 421,clip]{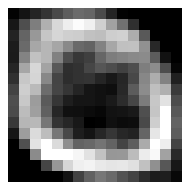} \\[-0.5ex]
        \rotatebox{90}{\makebox[\MACPlengthA][c]{SGD}} &
        \includegraphics[width=\MACPlengthA,bb=280 371 330 421,clip]{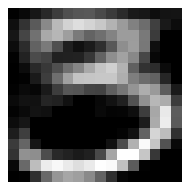} &
        \includegraphics[width=\MACPlengthA,bb=280 371 330 421,clip]{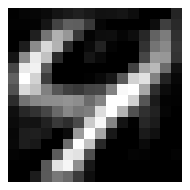} &  
        \includegraphics[width=\MACPlengthA,bb=280 371 330 421,clip]{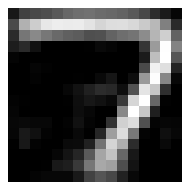} &
        \includegraphics[width=\MACPlengthA,bb=280 371 330 421,clip]{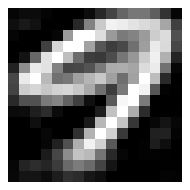} &
        \includegraphics[width=\MACPlengthA,bb=280 371 330 421,clip]{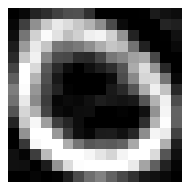}
      \end{tabular}
    \end{tabular}
    \caption{Training a deep autoencoder to reconstruct images of handwritten digits. \emph{Left}: nested function error~\eqref{e:nested} for each algorithm, with markers shown every iteration (MAC), every 100 iterations (CG), or every 20 epochs (SGD). For MAC/QP, we incremented the quadratic penalty $\mu$ as indicated at the red solid markers. All these experiments were run in the same computer using a single processor, except the parallel MAC training curve, which used 12 processors sharing the same memory using the Matlab Parallel Processing Toolbox. \emph{Right}: reconstruction of sample training images by different methods.}
    \label{f:usps}
  \end{center}
\end{figure}

\subsection{Heterogeneous training: radial basis functions autoencoder}
\label{s:expts:coil}

\begin{figure}[t]
  \begin{center}
    \begin{tabular}[c]{@{}c@{\hspace{0.08\linewidth}}c@{}}
      \begin{tabular}[c]{@{}c@{}}
        \psfrag{mu1}{{\scriptsize $\mu=1$}}
        \psfrag{mu5}{{\scriptsize $\mu=5$}}
        \psfrag{obj}[][]{objective function}
        \psfrag{time}[][]{runtime (hours)}
        \psfrag{MAC}{MAC}
        \psfrag{ALT. OPT}{Alt.~opt.}
        \psfrag{Parallel MAC}{Parallel MAC}
        \includegraphics[width=0.48\linewidth]{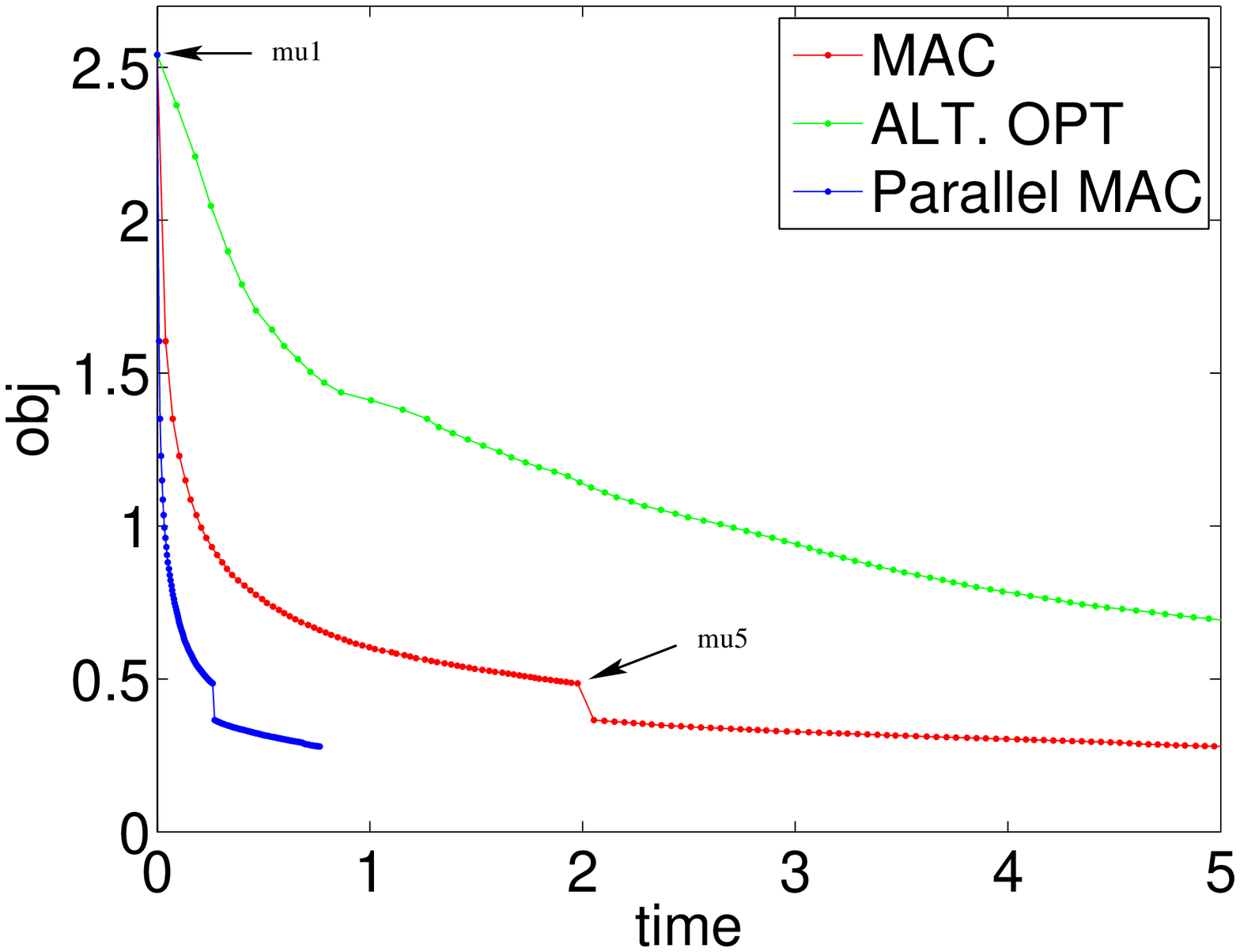}
      \end{tabular} &
      \small
      \begin{tabular}{@{}c@{\hspace{0\linewidth}}c@{\hspace{0\linewidth}}c@{\hspace{0\linewidth}}c@{\hspace{0\linewidth}}c@{\hspace{0\linewidth}}c@{}}
        \rotatebox{90}{\makebox[\MACPlengthA][c]{\caja[0.5]{c}{c}{Ground \\ truth}}} &
        \includegraphics[width=\MACPlengthA,bb=259 365 350 455,clip]{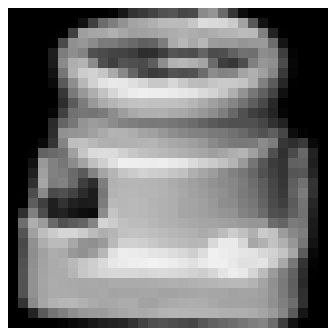} &
        \includegraphics[width=\MACPlengthA,bb=259 365 350 455,clip]{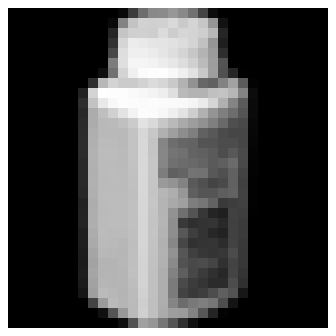} &  
        \includegraphics[width=\MACPlengthA,bb=259 365 350 455,clip]{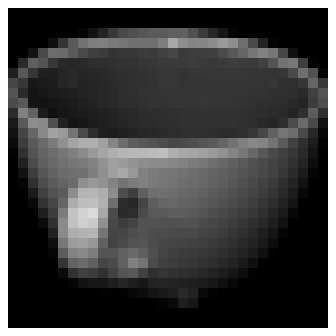} &
        \includegraphics[width=\MACPlengthA,bb=259 365 350 455,clip]{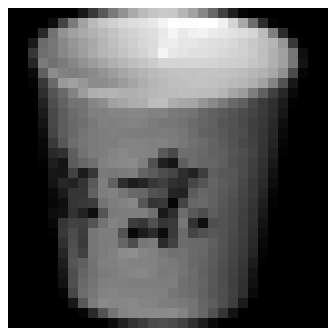} &
        \includegraphics[width=\MACPlengthA,bb=259 365 350 455,clip]{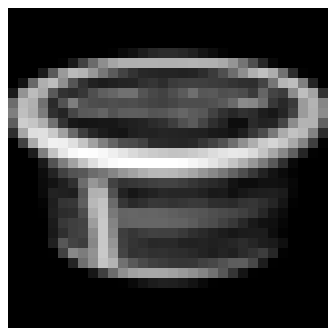} \\[-0.5ex]
        \rotatebox{90}{\makebox[\MACPlengthA][c]{\caja[0.5]{c}{c}{Initial \\ parameters}}} &
        \includegraphics[width=\MACPlengthA,bb=259 365 350 455,clip]{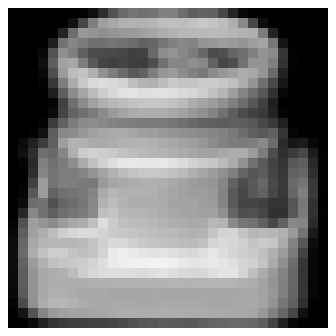} &
        \includegraphics[width=\MACPlengthA,bb=259 365 350 455,clip]{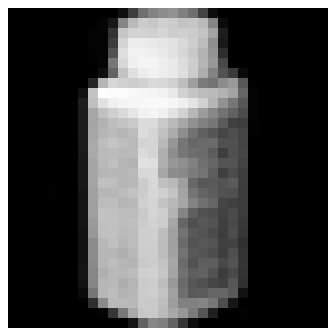} &  
        \includegraphics[width=\MACPlengthA,bb=259 365 350 455,clip]{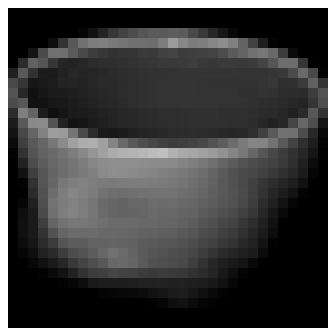} &
        \includegraphics[width=\MACPlengthA,bb=259 365 350 455,clip]{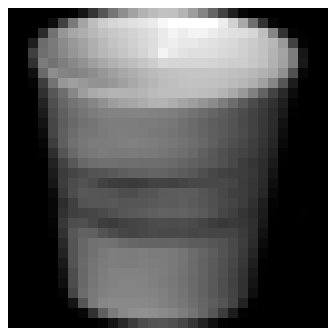} &
        \includegraphics[width=\MACPlengthA,bb=259 365 350 455,clip]{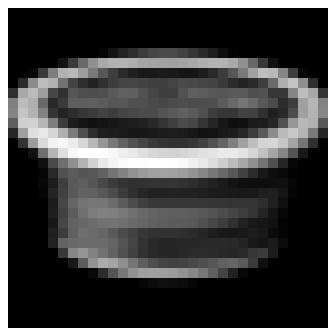} \\[-0.5ex]
        \rotatebox{90}{\makebox[\MACPlengthA][c]{MAC}} &
        \includegraphics[width=\MACPlengthA,bb=259 365 350 455,clip]{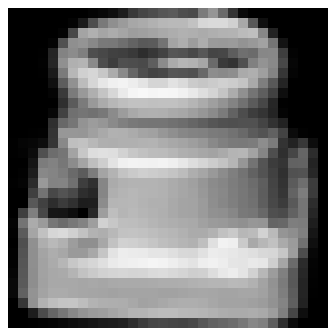} &
        \includegraphics[width=\MACPlengthA,bb=259 365 350 455,clip]{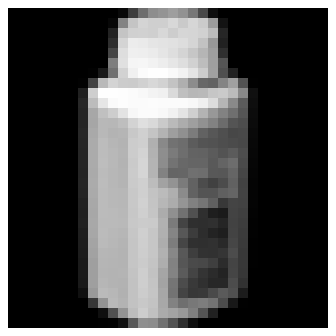} &  
        \includegraphics[width=\MACPlengthA,bb=259 365 350 455,clip]{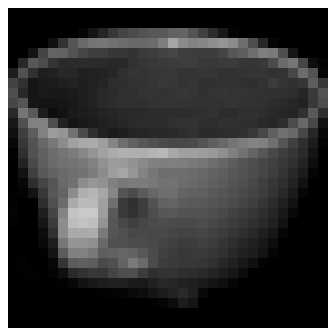} &
        \includegraphics[width=\MACPlengthA,bb=259 365 350 455,clip]{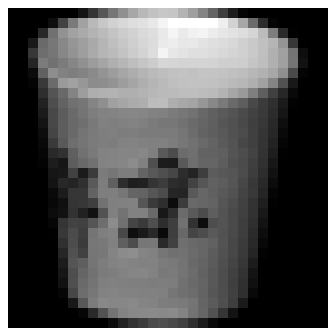} &
        \includegraphics[width=\MACPlengthA,bb=259 365 350 455,clip]{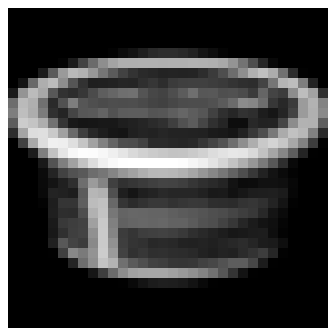} \\[-0.5ex]
        \rotatebox{90}{\makebox[\MACPlengthA][c]{Alt.~opt.}} &
        \includegraphics[width=\MACPlengthA,bb=259 365 350 455,clip]{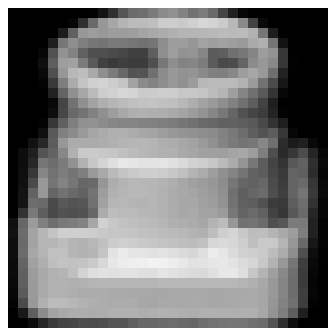} &
        \includegraphics[width=\MACPlengthA,bb=259 365 350 455,clip]{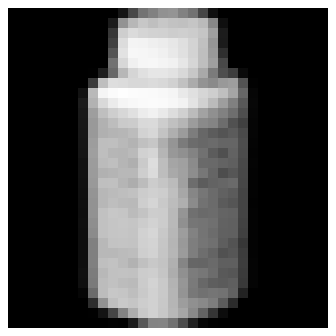} &  
        \includegraphics[width=\MACPlengthA,bb=259 365 350 455,clip]{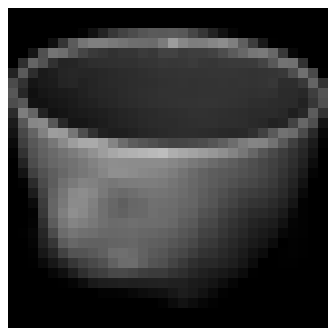} &
        \includegraphics[width=\MACPlengthA,bb=259 365 350 455,clip]{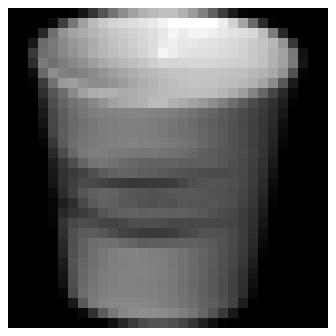} &
        \includegraphics[width=\MACPlengthA,bb=259 365 350 455,clip]{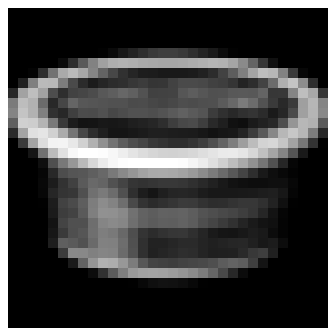}
      \end{tabular}
    \end{tabular}
    \caption{Training an RBF autoencoder to reconstruct images of objects. \emph{Left}: nested function error~\eqref{e:nested} for each algorithm, with markers shown every iteration (MAC, alternating optimization). Other details as in fig.~\ref{f:usps}. \emph{Right}: reconstruction of sample training images by different methods.}
    \label{f:coil}
  \end{center}
\end{figure}

\begin{figure}[t!]
  \begin{center}
    \begin{tabular}{@{}c@{}c@{}}
      \includegraphics[width=0.50\linewidth]{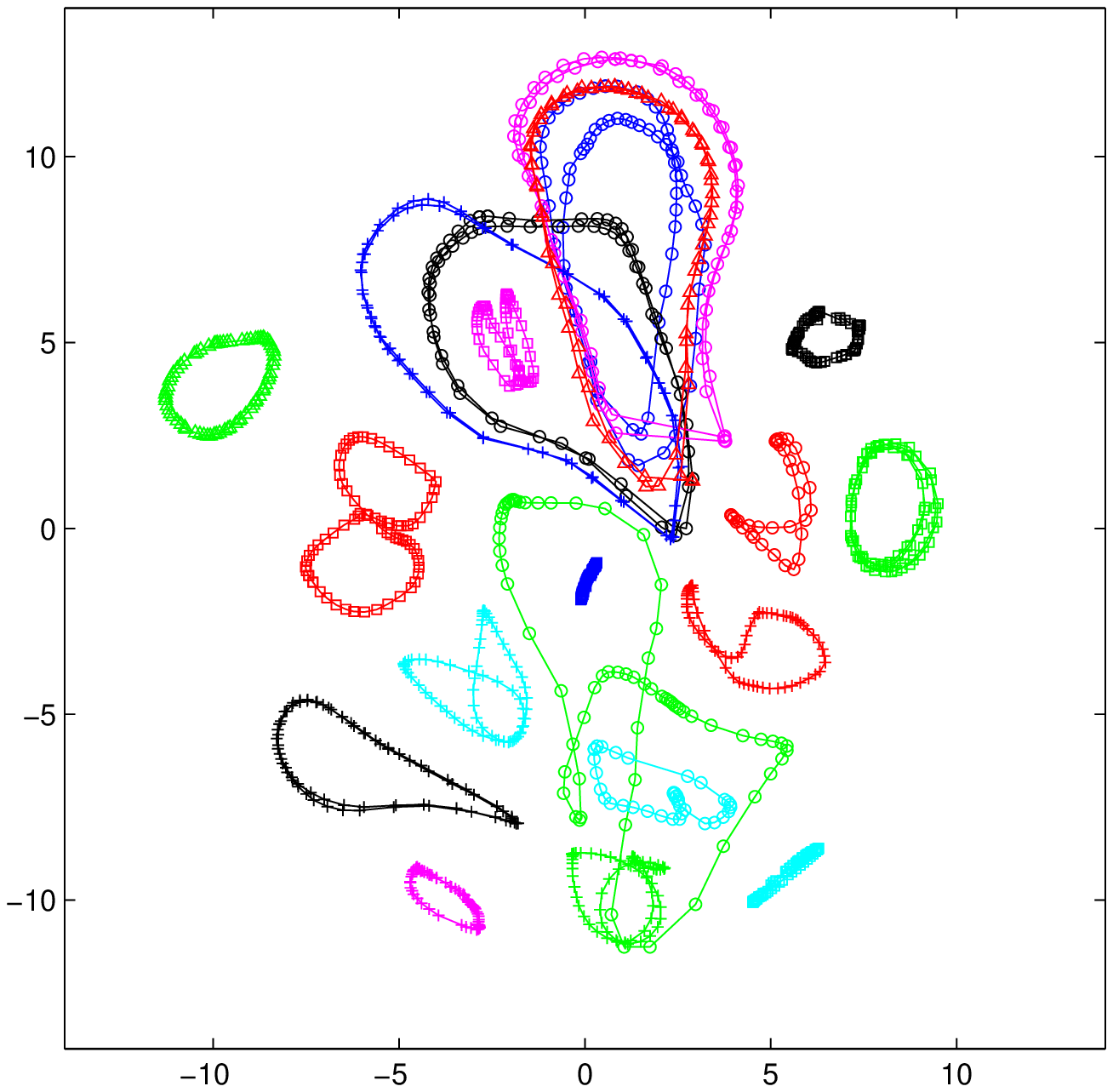} &
      \includegraphics[width=0.50\linewidth]{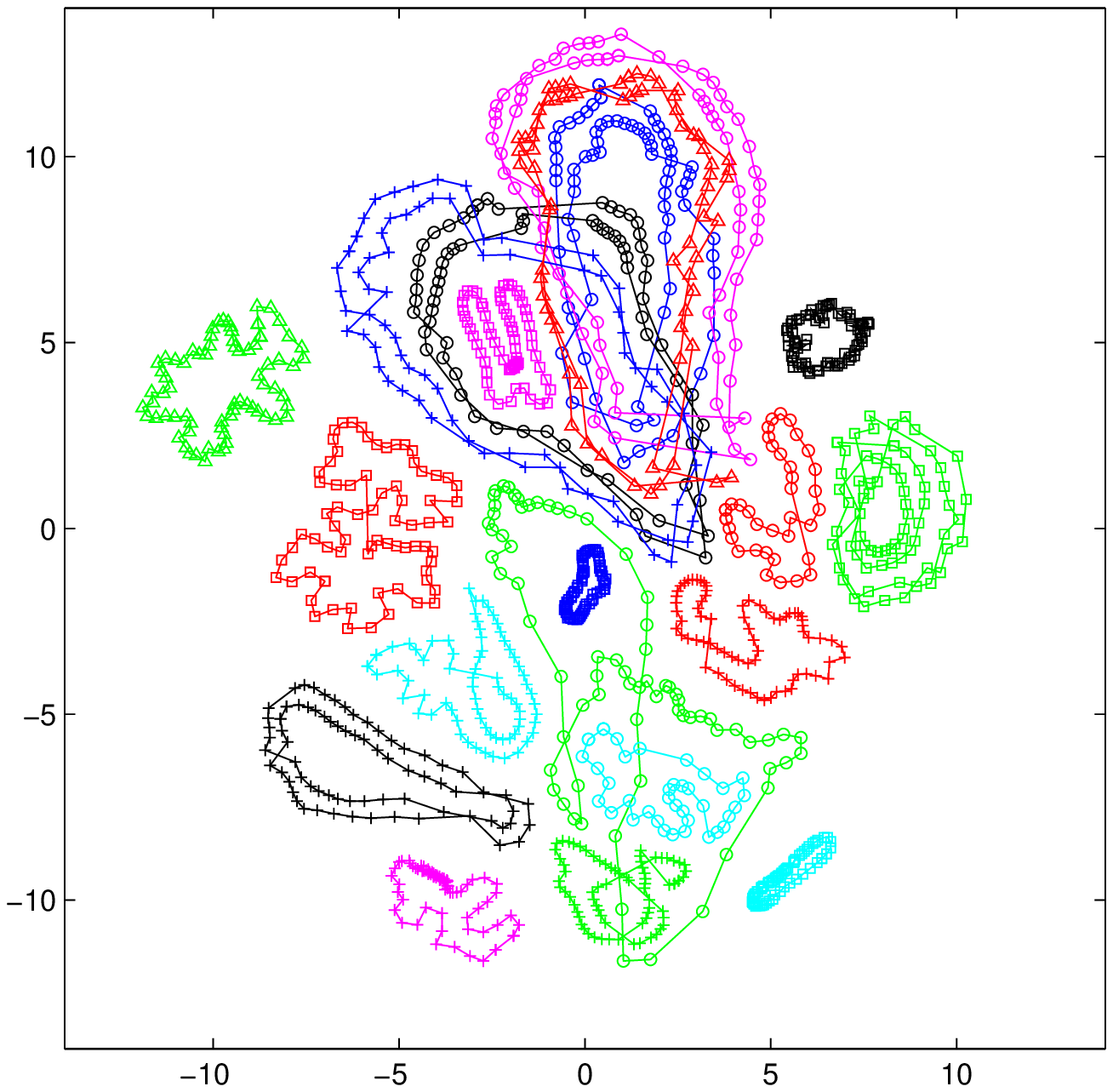}
    \end{tabular}
    \caption{Values of the \Z\ coordinates (2D latent space) for the COIL--20 data set. \emph{Left}: initialization obtained from the elastic embedding algorithm. \emph{Right}: results from MAC at the end of training (section~\ref{s:expts:coil}).}
    \label{f:coil_embeddings}
  \end{center}
\end{figure}

We use a dataset of object images to train an autoencoder architecture where both the encoder and decoder are RBF networks but, rather than using a gradient-based optimization for each subnet, we use $k$-means to train the basis functions of each RBF in the \W-step (while the weights in the remaining layers are trained by least-squares). We used MAC/QP introducing auxiliary coordinates only at the coding layer. Backprop-based algorithms are incompatible with $k$-means training, so instead we compare with alternating optimization. Fig.~\ref{f:coil} shows the learning curves.

The COIL--20 image dataset \citep{Nene_96a}, commonly used as a benchmark to test dimensionality reduction algorithms, contains rotation sequences of 20 different objects every 5 degrees (i.e., 72 images per object), each a grayscale image with pixel intensity in $[0,1]$. We resize the images to $32\times 32$. Thus, the data contain 20 closed, nonlinear 1D manifolds in a $1\,024$--dimensional space. We pick half of the images from objects 1 (duck) and 4 (cat) as validation set, which leaves a training set containing $N=1\,368$ images.

The autoencoder architecture is as follows. The bottleneck layer of low-dimensional codes has only 2 units, so that we can visualize the data. The encoder reduces the dimensionality of the input image to 2D, while the decoder reconstructs the image as best as possible from the 2D representation. Both the encoder and the decoder are radial basis function (RBF) networks, each having a single hidden layer. The first one (encoder) has the form $\z = \f_2(\f_1(\x;\W_1);\W_2) = \W_2 \, \f_1(\x;\W_1)$, where the vector $\f_1(\x;\W_1)$ has $M_1 = 1\,368$ elements (basis functions) of the form $\exp{(\smash{-\norm{(\x-\w_{1i})/\sigma_1}}^2)}$, $i=1,\dots,M_1$, with $\sigma_1 = 4$, and maps an image \x\ to a 2D space \z. The second one (decoder) has the form $\x' = \f_4(\f_3(\z;\W_3);\W_4) = \W_4 \, \f_3(\z;\W_3)$, where the vector $\f_3(\z;\W_3)$ has $M_3 = 1\,368$ elements (basis functions) of the form $\exp{(\smash{-\norm{(\x-\w_{3i})/\sigma_3}}^2)}$, $i=1,\dots,M_3$, with $\sigma_3 = 0.5$, and maps a 2D point \z\ to a $1\,024$D image. Thus, the complete autoencoder is the concatenation of the two Gaussian RBF networks, it has $K=3$ hidden layers with sizes 1\,024--1\,368--2--1\,368--1\,024, and a total of almost 3 million weights. As is usual with RBF networks, we applied a quadratic regularization to the linear-layer weights with a small value ($\lambda_2 = \lambda_4 = 10^{-3}$). The nested problem is then to minimize the following objective function, which is a least-squares error plus a quadratic regularization on the linear-layer weights:
\begin{equation}
  \label{e:nested-rbf}
  \hspace{-1.5ex}E_1(\W) = \frac{1}{2} \sum^N_{n=1}{\norm{\y_n - \f(\x_n;\W)}^2} + \lambda_2 \norm{\W_2}^2_F + \lambda_4 \norm{\W_4}^2_F \quad \f(\x;\W) = \f_4(\f_3(\f_2(\f_1(\x;\W_1);\W_2);\W_3);\W_4).\hspace{-1ex}
\end{equation}
In practice, RBF networks are trained in two stages \citep{Bishop06a}. Consider the encoder, for example. First, one trains the centers $\W_1$ using a clustering algorithm applied to the inputs $\{\x_n\}^N_{n=1}$, typically $k$-means or (when the number of centers is large) simply by fixing them to be a random subset of the inputs. Second, having determined $\W_1$, one obtains $\W_2$ from a linear least-squares problem, by solving a linear system. The reason why this is preferred to a fully nonlinear optimization over centers $\W_1$ and weights $\W_2$ is that it achieves near-optimal nets with a simple, noniterative procedure. This type of two-stage noniterative strategy to obtain nonlinear networks is widely applied beyond RBF networks, for example with support vector machines \citep{SchoelSmola01a}, kernel PCA \citep{Schoel_98a}, slice inverse regression \citep{Li91a} and others.

We wish to capitalize on this attractive property to train deep autoencoders constructed by concatenating RBF networks. However, backprop-based algorithms are incompatible with this two-stage training procedure, since it does not use derivatives to optimize over the centers. This leads us to the two following optimization methods: an alternating optimization approach, and MAC.

We can use an alternating optimization approach where we alternate the following two steps: (1) A step where we fix $(\W_1,\W_2)$ and train $(\W_3,\W_4)$, by applying $k$-means to $\W_3$ and a linear system to $\W_4$. This step is identical to the \W-step in MAC over $(\W_1,\W_2)$. (2) A step where we fix $(\W_3,\W_4)$ and train $(\W_1,\W_2)$, by applying $k$-means to $\W_1$ and a nonlinear optimization to $\W_2$ (we use nonlinear conjugate gradients with 10 steps). This is because $\W_2$ no longer appears linearly in the objective function, but is nonlinearly embedded as the argument of the decoder. This step is significantly slower than the \W-step in MAC over $(\W_3,\W_4)$.

We define the MAC-constrained problem as follows. We introduce auxiliary coordinates only at the coding layer (rather than at all $K=3$ hidden layers). This allows the \W-step to become the desired $k$-means plus linear system training for the encoder and decoder separately. It requires no programming effort; we simply call an existing, $k$-means-based RBF training algorithm for each of the encoder and decoder separately. We start with a quadratic penalty parameter $\mu = 1$ and increase it to $\mu = 5$ after 70 iterations.

Since we use as many centers as data points ($M_1 = M_3 = N$), the $k$-means step simplifies (for both methods) to setting each basis function center to an input point.

In this experiment, instead of using random initial weights, we obtained initial values for the \Z\ coordinates by running a nonlinear dimensionality reduction method, the elastic embedding (EE) \citep{Carreir10a}; this gives significantly better embeddings than spectral methods; \citep{Tenenb_00a,RoweisSaul00a}. This takes as input a matrix of $N \times N$ similarity values between every pair of COIL images $\x_1,\dots,\x_N$, and produces a nonlinear projection $\z_n$ in 2D for each $\x_n$. We used Gaussian similarities with a kernel width of $10$ and run EE for 200 iterations using $\lambda=100$ as its user parameter. All the optimization algorithms were initialized from these projections.

Fig.~\ref{f:coil}(left) shows the nested function error~\eqref{e:nested-rbf} (normalized per data point, i.e., divided by $N$). As before, MAC/QP achieves a large error decrease in a few iterations. Alternating optimization is much slower. Again, a parallel implementation of MAC/QP achieves a large speedup, which is nearly linear on the number of processors (fig.~\ref{f:parallel}). Fig.~\ref{f:coil}(right) shows some COIL images and their reconstruction at the end of training for each algorithm. Fig.~\ref{f:coil_embeddings} shows the initial projections \Z\ (from the elastic embedding algorithm) and the final projections \Z\ (after running MAC/QP). Most of the manifolds have improved, for example opening up loops that were folded.

\subsection{Learning the architecture: RBF autoencoder}
\label{s:expts:coil-model-selection}

\begin{figure}[t]
  \begin{center}
    \psfrag{mu1}[][]{{\tiny $\mu=1$}}
    \psfrag{mu5}[][]{{\tiny $\mu=5$}}
    \psfrag{arch1}[][]{{\tiny $(1\,368,1\,368)$}}
    \psfrag{arch2}[][]{{\tiny $(700,150)$}}
    \psfrag{arch3}[][]{{\tiny $(1\,050,150)$}}
    \psfrag{arch4}[][]{{\tiny $(1\,368,150)$}}
    \psfrag{obj}[][]{objective function}
    \psfrag{its}[][]{iteration}
    \includegraphics[width=0.48\linewidth]{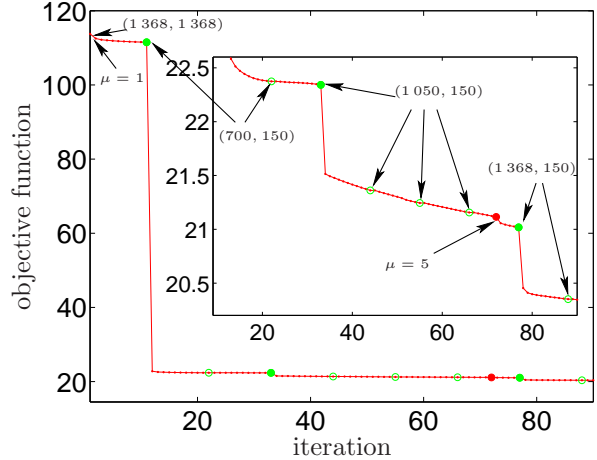}
    \caption{Learning the architecture of the RBF autoencoder for the dataset of fig.~\ref{f:coil} using MAC. We show the total error $E_1(\W) + C(\W)$ (the nested function error~\eqref{e:nested} plus the model cost) per point. Model selection steps are run every 10 iterations and are indicated with green markers (solid if the architecture changes and empty if it does not change). Other details as in fig.~\ref{f:coil}.}
    \label{f:coil-model-selection}
  \end{center}
\end{figure}

\begin{figure}[t!]
  \begin{center}
    \psfrag{labs}[t][]{number of processors}
    \psfrag{speedup}[][t]{speedup}
    \psfrag{usps}{{\scriptsize USPS data}}
    \psfrag{coil}{{\scriptsize COIL data}}
    \psfrag{coil architecture}{{\scriptsize COIL (learn arch.)}}
    \begin{tabular}[c]{@{}c@{}}
      \includegraphics[width=0.4\linewidth]{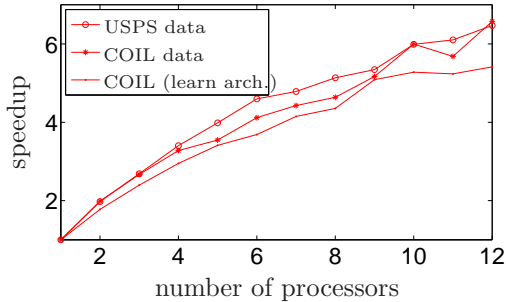}
    \end{tabular}
    \caption{Parallel processing speedup of MAC/QP as a function of the number of processors for each of the experiments of figures~\ref{f:usps}, \ref{f:coil} and~\ref{f:coil-model-selection}. The speedup in all our experiments was approximately linear, reaching about $6\times$ with $12$ processors.}
    \label{f:parallel}
  \end{center}
\end{figure}

We repeat the experiment of the RBF autoencoder of the previous section, but now we learn its architecture. We jointly learn the architecture of the encoder and of the decoder by trying $50$ different values for the number of basis functions in each (a search space of $50^2 = 2\,500$ architectures). We define the following objective function over architectures and their weights:
\begin{equation}
  \label{e:nested-rbf-model-selection}
  \bar{E}(\W) = E_1(\W) + C(\W)
\end{equation}
where $E_1(\W)$ is the nested error from eq.~\eqref{e:nested-rbf} (including regularization terms), and $C(\W) = C(\W_1) + \dots + C(\W_4)$ is the model selection term. We use the well-known AIC criterion \citep{Hastie_09a}. This is defined as
\begin{equation}
  \label{e:aic}
  \bar{E}(\bTheta) = \text{SSE}(\bTheta) + C(\bTheta) \qquad C(\bTheta) = 2 \epsilon^2 \abs{\bTheta}
\end{equation}
(times a constant $\frac{1}{N}$, which we omit) where SSE$(\bTheta)$ is the sum of squared errors achieved in the training set with a model having parameters $\bTheta$, $\epsilon^2$ is the mean squared error achieved by a low-bias model (typically estimated by training a model with many parameters) and $\abs{\bTheta}$ is the number of free parameters in the model. In our case, this means we use $C(\W)$ defined as
\begin{equation}
  \label{e:aic-rbf}
  C(\W) = 2 \epsilon^2 \abs{\W} = 2 \epsilon^2 ( D M_1 + L M_1 + L M_3 + D M_3) = 2 \epsilon^2 (D+L)(M_1+M_3)
\end{equation}
where $M_1$ and $M_3$ are the numbers of centers for the encoder and decoder, respectively (first and third hidden layers), and $D = 1\,024$ and $L = 2$ are the input and output dimension of the encoder, respectively (equivalently, the output and input dimension of the decoder). The total number of free parameters (centers and linear weights) in the autoencoder is thus $\abs{\W} = (D+L)(M_1+M_3)$.

We choose each of the numbers of centers $M_1$ and $M_3$ from a discrete set consisting of the $50$ equispaced values in the range $150$ to $1\,368$ (a total of $50^2 = 2\,500$ different architectures). We estimated $\epsilon^2 = 0.05$ from the result of the RBF autoencoder of section~\ref{s:expts:coil}, which had a large number of parameters and thus a low bias. As in that section, the centers of each network are constrained to be equal to a subset of the input points (chosen at random). We set $\sigma_1 = 4$, $\sigma_3 = 2.5$ and $\lambda_1 = \lambda_2 = 10^{-3}$. We start the MAC/QP optimization from the most complex model, having $M_1 = M_3 = 1\,368$ centers (i.e., the model of the previous section). While every iteration optimizes the MAC/QP objective function~\eqref{e:mac-quadpen} over $(\W,\Z)$, we run a model selection step only every 10 iterations. This selects separately for each net the best $M_k$ value and potentially changes the size of \W. Thus, every 11th iteration is a model selection step, which may or may not change the architecture.

Figure~\ref{f:coil-model-selection} shows the total error $\bar{E}(\W) = E_1(\W) + C(\W)$ of eq.~\eqref{e:nested-rbf-model-selection} (the nested function error plus the model cost). Model selection steps are indicated with green markers (solid if the architecture did change and empty if it did not change), annotated with the resulting value of $(M_1,M_3)$. Other details are as in fig.~\ref{f:coil}. The first change of architecture moves to a far smaller model $(M_1,M_3) = (700,150)$, achieving an enormous decrease in objective. This is explained by the strong penalty that AIC imposes on the number of parameters, favoring simpler models. Then, this is followed by a few minor changes of architecture interleaved with a continuous optimization of its weights. The final architecture has $(M_1,M_3) = (1\,368,150)$, for a total of $1.5$ million weights. While this architecture incurs a larger training error than that of the previous section, it uses a much simpler model and has a lower value for the overall objective function of eq.~\eqref{e:nested-rbf-model-selection}. Because, early during the optimization, MAC/QP settles on an architecture that is quite smaller than the one used in fig.~\ref{f:coil}, the result is in fact achieved in even less time. And, again, the parallel implementation is trivial and achieves an approximately linear speedup on the number of processors (fig.~\ref{f:parallel}).

\section{Conclusion}

MAC drastically facilitates, in runtime and human effort, the practical design and estimation of nonconvex, nested problems by jointly optimizing over all parameters, reusing existing algorithms, searching automatically over architectures and affording massively parallel computation, while provably converging to a solution of the nested problem. It could replace or complement backpropagation-based algorithms in learning nested systems both in the serial and parallel settings. It is particularly timely given that serial computation is reaching a plateau and cloud computing is becoming a commodity, and intelligent data processing is finding its way into mainstream devices (phones, cameras, etc.), thanks to increases in computational power and data availability. An important area of application may be the joint, automatic tuning of all stages of a complex, intelligent-processing system in data-rich disciplines, such as those found in computer vision and speech, in a distributed cloud computing environment. MAC also opens many questions, such as the optimal way to introduce auxiliary coordinates in a given problem, and the choice of specific algorithms to optimize the \W- and \Z-steps.

\section{Acknowledgments}

Work funded in part by NSF CAREER award IIS--0754089.

\appendix

\section{Theorems and proofs}
\label{s:proofs}

\subsection{Definitions}
\label{s:proofs:def}

Consider a regression problem of mapping inputs \x\ to outputs \y\ (both high-dimensional) with a deep net $\f(\x)$ given a dataset of $N$ pairs $(\x_n,\y_n)$. We define the \emph{nested objective function} to learn a deep net with $K$ hidden layers, like that in fig.~\ref{f:deepnet}, as (to simplify notation, we ignore bias parameters and assume each hidden layer has $H$ units):
\begin{equation}
%  \label{e:nested}
  E_1(\W) = \frac{1}{2} \sum^N_{n=1}{\norm{\y_n - \f(\x_n;\W)}^2} \qquad \f(\x;\W) = \f_{K+1}(\dots \f_2(\f_1(\x;\W_1);\W_2)\dots;\W_{K+1}) \tag{\ref{e:nested}}
\end{equation}
where each layer function has the form $\f_k(\x;\W_k) = \sigma(\W_k\x)$, i.e., a linear mapping followed by a squashing nonlinearity ($\sigma(t)$ applies a scalar function, such as the sigmoid $1/(1+e^{-t})$, elementwise to a vector argument, with output in $[0,1]$).

In the \emph{method of auxiliary coordinates (MAC)}, we introduce one auxiliary variable per data point and per hidden unit (so $\Z = (\Z_1,\dots,\Z_N)$, with $\z_n = (\z_{1,n},\dots,\z_{K,n})$) and define the following equality-constrained optimization problem:
\begin{equation}
%  \label{e:mac}
  E(\W,\Z) = \frac{1}{2} \sum^N_{n=1}{\norm{\y_n - \f_{K+1}(\z_{K,n};\W_{K+1})}^2} \text{ s.t.\ }
  \renewcommand{\arraystretch}{0.5}
  \left\{
  \begin{array}{@{}l@{}}
    \z_{K,n} = \f_K(\z_{K-1,n};\W_K) \\ \dots \\ \z_{1,n} = \f_1(\x_n;\W_1)
  \end{array}
  \right\} n=1,\dots,N.\tag{\ref{e:mac}}
  \hspace{-3ex}
\end{equation}
Sometimes, for notational convenience (in particular in theorem~\ref{th:MACQP}), we will write the constraints for the $n$th point as a single vector constraint $\z_n - \F(\z_n,\W;\x_n) = \0$ (with an obvious definition for \F). We will also call $\Omega$ the feasible set of the MAC-constrained problem, i.e.,
\begin{equation}
  \label{e:feasible}
  \Omega = \{(\W,\Z)\mathpunct{:}\ \z_n = \F(\z_n,\W;\x_n),\ n=1,\dots,N\}.
\end{equation}

To solve the constrained problem~\eqref{e:mac} using the quadratic-penalty (QP) method \citep{NocedalWright06a}, we optimize the following function over $(\W,\Z)$ for fixed $\mu>0$ and drive $\mu \rightarrow \infty$:
\begin{equation}
%  \label{e:mac-quadpen}
  E_Q(\W,\Z;\mu) = \frac{1}{2} \sum^N_{n=1}{\norm{\y_n - \f_{K+1}(\z_{K,n};\W_{K+1})}^2} + \frac{\mu}{2} \sum^N_{n=1}{\sum^K_{k=1}{\norm{\z_{k,n} - \f_k(\z_{k-1,n};\W_k)}^2}}.\tag{\ref{e:mac-quadpen}}
\end{equation}

\subsection{Equivalence of the MAC and nested formulations}
\label{s:proofs:equiv}

First, we give a theorem that holds under very general assumptions. In particular, it does not require the functions to be smooth, it holds for any loss function beyond the least-squares one, and it holds if the nested problem is itself subject to constraints. % In fact, the function f need not even be continuous.
\begin{thm}
  \label{th:MAC-nested-equiv}
  The nested problem~\eqref{e:nested} and the MAC-constrained problem~\eqref{e:mac} are equivalent in the sense that their minimizers are in a one-to-one correspondence.
\end{thm}
\begin{proof}
  Let us prove that any minimizer of the nested problem is associated with a unique minimizer of the MAC-constrained problem $(\Rightarrow)$, and vice versa $(\Leftarrow)$. Recall the following definitions~\citep{NocedalWright06a}: (i) For an unconstrained minimization problem $\min_{\x\in\bbR^n}{F(\x)}$, $\x^*\in\bbR^n$ is a local minimizer if there exists a nonempty neighborhood $\calN\subset\bbR^n$ of $\x^*$ such that $F(\x^*) \le F(\x)\ \forall \x \in \calN$. (ii) For a constrained minimization problem $\min{F(\x)}$ s.t.\ $\x\in\Omega\subset\bbR^n$, $\x^*\in\bbR^n$ is a local minimizer if $\x^*\in\Omega$ and there exists a nonempty neighborhood $\calN\subset\bbR^n$ of $\x^*$ such that $F(\x^*) \le F(\x)\ \forall \x \in \calN \cap \Omega$.

  Define the ``forward-propagation'' function $\g(\W)$ as the result of mapping $\z_{1,n} = \f_1(\x_n;\W_1),\dots,\z_{K,n} = \f_K(\z_{K-1,n};\W_K)$ for $n=1,\dots,N$. This maps each \W\ to a unique \Z, and satisfies $\f_{K+1}(\z_{K,n};\W_{K+1}) = \f_{K+1}(\dots \f_2(\f_1(\x_n;\W_1);\W_2)\dots;\W_{K+1}) = \f(\x_n;\W)$ for $n=1,\dots,N$, and therefore that $E_1(\W) = E(\W,\g(\W))$ for any \W.

  $(\Rightarrow)$ Let $\W^*$ be a local minimizer of the nested problem~\eqref{e:nested}. Then, there exists a nonempty neighborhood \calN\ of $\W^*$ such that $E_1(\W^*) \le E_1(\W)\ \forall \W \in \calN$. Let $\Z^* = \g(\W^*)$ and call $\calM = \{(\W,\Z)\mathpunct{:}\ \W\in\calN\ \text{and}\ \Z=\g(\W)\}$, which is a nonempty neighborhood of $(\W^*,\Z^*)$ in $(\W,\Z)$-space. Now, for any $(\W,\Z) \in \calM \cap \calN$ we have that $E(\W,\Z) = E(\W,\g(\W)) = E_1(\W) \ge E_1(\W^*) = E(\W^*,\g(\W^*)) = E(\W^*,\Z^*)$. Hence $(\W^*,\Z^*)$ is a local minimizer of the MAC-constrained problem. % Strictly, I should take as neighborhood calM = calN x {all-Z}, which is open, and whose intersection with Omega gives the calM I actually use.

  $(\Leftarrow)$ Let $(\W^*,\Z^*)$ be a local minimizer of the MAC-constrained problem~\eqref{e:mac}. Then, there exists a nonempty neighborhood \calM\ of $(\W^*,\Z^*)$ such that $E(\W^*,\Z^*) \le E(\W,\Z)\ \forall (\W,\Z) \in \calM \cap \Omega$. Note that $(\W,\Z) \in \calM \cap \Omega \Rightarrow \Z = \g(\W) \Rightarrow E(\W,\Z) = E_1(\W)$, and this applies in particular to $(\W^*,\Z^*)$ (which, being a solution, is feasible and thus belongs to $\calM \cap \Omega$). Calling $\calN = \{\W\mathpunct{:}\ (\W,\Z) \in \calM \cap \Omega\}$, we have that $\forall\W\in\calN \mathpunct{:}\ E_1(\W) = E(\W,\g(\W)) = E(\W,\Z) \ge E(\W^*,\Z^*) = E(\W^*,\g(\W^*)) = E_1(\W^*)$. Hence $\W^*$ is a local minimizer of the nested problem.

  Finally, one can see that the proof holds if the nested problem uses a loss function that is not the least-squares one, and if the nested problem is itself subject to constraints.
\end{proof}
Obviously, the theorem holds if we exchange $\ge$ with $>$ everywhere (thus exchanging non-strict with strict minimizers), and if we exchange ``min'' with ``max'' (hence the maximizers of the MAC and nested formulations are in a one-to-one correspondence as well). Figure~\ref{f:MAC-nested-equiv} illustrates the theorem. Essentially, the nested objective function $E_1(\W)$ stretches along the manifold defined by $(\W,\Z = \g(\W))$ preserving the minimizers and maximizers. The projection on \W-space of the part of $E(\W,\Z)$ that sits on top of that manifold recovers $E_1(\W)$.

\begin{figure}[t]
  \begin{center}
    \psfrag{z}[b][t]{\Z}
    \psfrag{w1}[b][t]{\W}
    \psfrag{E}[][t]{$E(\W,\Z)$}
    \psfrag{Ef}{$E(\W,\g(\W))$}
    \psfrag{E1}{\textcolor{blue}{$E_1(\W)$}}
    \psfrag{g}[b][b]{\textcolor{red}{$(\W,\g(\W))$}}
    \psfrag{minE1}{$\W^*$}
    \psfrag{minE}{$(\W^*,\Z^*)$}
    \includegraphics[width=0.4\linewidth]{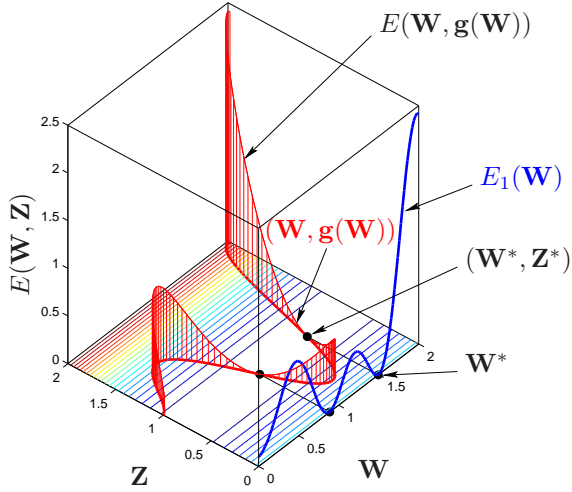}
    \caption{Illustration of the equivalence between the nested and MAC-constrained problems (see the proof of theorem~\ref{th:MAC-nested-equiv}). The MAC objective function $E(\W,\Z)$ is shown with contour lines in the $(\W,\Z)$-space, and with the vertical red lines on the feasible set $(\W,\g(\W))$. The nested objective function $E_1(\W)$ is shown in blue. Corresponding minima for both problems, $\W^*$ and $(\W^*,\Z^*)$, are indicated.}
    \label{f:MAC-nested-equiv}
  \end{center}
\end{figure}

\subsection{KKT conditions}

We now show that the first-order necessary (Karush-Kuhn-Tucker, KKT) conditions of both problems (nested and MAC-constrained) have the same stationary points. For simplicity and clarity of exposition, we give a proof for the special case of $K=1$. The proof for $K>1$ layers follows analogously. We assume the functions $\f_1$ and $\f_2$ have continuous first derivatives w.r.t.\ both its input and its weights. $\J_{\f_2}(\cdot;\W_2)$ indicates the Jacobian of $\f_2$ w.r.t.\ its input. To simplify notation, we sometimes omit the dependence on the weights; for example, we write $\f_2(\f_1(\x;\W_1);\W_2)$ as $\f_2(\f_1(\x))$, and $\J_{\f_2}(\cdot;\W_2)$ as $\J_{\f_2}(\cdot)$.
\begin{thm}
  \label{th:MAC-nested-equiv-KKT}
  The KKT conditions for the nested problem~\eqref{e:nested} and the MAC-constrained problem~\eqref{e:mac} are equivalent.
\end{thm}
\begin{proof}
  The nested problem for a nested function $\f_2(\f_1(\x))$ is:
  \begin{equation*}
    \min_{\W_1,\W_2}{ E_1(\W_1,\W_2) = \frac{1}{2} \sum^N_{n=1}{\norm{\y_n-\f_2(\f_1(\x_n;\W_1);\W_2)}^2} }.
  \end{equation*}
  Then we have the stationary point equation (first-order necessary conditions for a minimizer):
  \begin{align}
    \label{e:stat-pt1}
    \frac{\partial E_1}{\partial\W_1} &= -\sum^N_{n=1}{\frac{\partial\f_1^T}{\partial\W_1}(\x_n) \, \J_{\f_2}(\f_1(\x_n))^T (\y_n-\f_2(\f_1(\x_n)))} = \0 \\
    \label{e:stat-pt2}
    \frac{\partial E_1}{\partial\W_2} &= -\sum^N_{n=1}{\frac{\partial\f_2^T}{\partial\W_2}(\f_1(\x_n)) \, (\y_n-\f_2(\f_1(\x_n))) } = \0
  \end{align}
  which is satisfied by all the minima, maxima and saddle points.

  The MAC-constrained problem is
  \begin{equation*}
    \min_{\W_1,\W_2,\Z}{ E(\W_1,\W_2,\Z) = \frac{1}{2} \sum^N_{n=1}{\norm{\y_n-\f_2(\z_n;\W_2)}^2} } \text{ s.t.\ } \z_n = \f_1(\x_n;\W_1),\ n=1,\dots,N,
  \end{equation*}
  with Lagrangian
  \begin{equation*}
    \calL(\W_1,\W_2,\Z,\blambda) = \frac{1}{2} \sum^N_{n=1}{\norm{\y_n-\f_2(\z_n;\W_2)}^2} - \sum^N_{n=1}{\blambda^T_n (\z_n-\f_1(\x_n;\W_1))}
  \end{equation*}
  and KKT conditions
  \begin{align}
    \label{e:kkt1}
    \frac{\partial\calL_1}{\partial\W_1} &= \sum^N_{n=1}{ \frac{\partial\f_1^T}{\partial\W_1}(\x_n) \, \blambda_n } = \0 \\
    \label{e:kkt2}
    \frac{\partial\calL_1}{\partial\W_2} &= -\sum^N_{n=1}{ \frac{\partial\f_2^T}{\partial\W_2}(\f_1(\x_n)) \, (\y_n-\f_2(\z_n)) } = \0 \\
    \label{e:kkt3}
    \frac{\partial\calL_1}{\partial\z_n} &= -\J_{\f_2}(\z_n)^T (\y_n-\f_2(\z_n)) - \blambda_n = \0,\ n=1,\dots,N \\
    \label{e:kkt4}
    \z_n &= \f_1(\x_n;\W_1),\ n=1,\dots,N.
  \end{align}
  Substituting $\blambda_n$ from eq.~\eqref{e:kkt3} and $\z_n$ from eq.~\eqref{e:kkt4}:
  \begin{align}
    \label{e:kkt3b}
    \blambda_n &= -\J_{\f_2}(\z_n)^T (\y_n-\f_2(\z_n)),\ n=1,\dots,N \tag{\ref{e:kkt3}'} \\
    \label{e:kkt4b}
    \z_n &= \f_1(\x_n;\W_1),\ n=1,\dots,N \tag{\ref{e:kkt4}'}
  \end{align}
  into eqs.~\eqref{e:kkt1}--\eqref{e:kkt2} we recover eqs.~\eqref{e:stat-pt1}--\eqref{e:stat-pt2}, thus a KKT point of the constrained problem is a stationary point of the nested problem. Conversely, given a stationary point $(\W_1,\W_2)$ of the nested problem, and defining $\blambda_n$ and $\z_n$ as in eqs.~\eqref{e:kkt3b}--\eqref{e:kkt4b}, then $(\W_1,\W_2,\Z,\blambda)$ satisfies eqs.~\eqref{e:kkt1}--\eqref{e:kkt4} and so is a KKT point of the constrained problem. Hence, there is a one-to-one correspondence between the stationary points of the nested problem and the KKT points of the MAC-constrained problem.
\end{proof}
From theorem~\ref{th:MAC-nested-equiv} and~\ref{th:MAC-nested-equiv-KKT}, it follows that the minimizers, maximizers and saddle points of the nested problem are in one-to-one correspondence with the respective minimizers, maximizers and saddle points of the MAC-constrained problem. % No need to check the second-order conditions.

\section{Convergence of the quadratic-penalty method for MAC}
\label{s:proofs:QP}

Let us first give convergence conditions for the general equality-constrained minimization problem:
\begin{equation}
  \label{e:eqconstr}
  \min{ f(\x) } \text{ s.t.\ } c_i(\x) = 0,\ i=1,\dots,m
\end{equation}
and the quadratic-penalty (QP) function
\begin{equation}
  \label{e:eqconstrQ}
  Q(\x;\mu) = f(\x) + \frac{\mu}{2} \sum^m_{i=1}{c^2_i(\x)}
\end{equation}
with penalty parameter $\mu>0$. Given a positive increasing sequence $(\mu_k) \rightarrow \infty$, a nonnegative sequence $(\tau_k) \rightarrow 0$, and a starting point $\x_0$, the QP method finds an approximate minimizer $\x_k$ of $Q(\x;\mu_k)$ for $k=1,2,\dots$, so that the iterate $\x_k$ satisfies $\norm{\nabla_{\x}{Q(\x_k;\mu_k)}} \le \tau_k$. Given this algorithm, we have the following theorems:
\begin{thm}[\protect{\citealp[Th.~17.1]{NocedalWright06a}}]
  \label{th:17_1}
  Suppose that $(\mu_k) \rightarrow \infty$ and $(\tau_k) \rightarrow 0$. If each $\x_k$ is the exact global minimizer of $Q(\x;\mu_k)$, then every limit point $\x^*$ of the sequence $(\x_k)$ is a global solution of the problem~\eqref{e:eqconstr}.
\end{thm}
\begin{thm}[\protect{\citealp[Th.~17.2]{NocedalWright06a}}]
  \label{th:17_2}
  Suppose that $(\mu_k) \rightarrow \infty$ and $(\tau_k) \rightarrow 0$, and that $\x^*$ is a limit point of $(\x_k)$. Then $\x^*$ is a stationary point of the function $\sum^m_{i=1}{c^2_i(\x)}$. Besides, if the constraint gradients $\nabla{c_i(\x^*)},\ i=1,\dots,m$ are linearly independent, then $\x^*$ is a KKT point for the problem~\eqref{e:eqconstr}. For such points, we have for any infinite subsequence \calK\ such that $\lim_{k\in\calK}{\x_k} = \x^*$ that $\lim_{k\in\calK}{-\mu_k c_i(\x_k)} = \lambda^*_i,\ i=1,\dots,m$, where $\blambda^*$ is the multiplier vector that satisfies the KKT conditions for the problem~\eqref{e:eqconstr}.
\end{thm}
If now we particularize these general theorems to our case, we can obtain stronger theorems. Theorem~\ref{th:17_1} is generally not applicable, because optimization problems involving nested functions are typically not convex and have local minima. Theorem~\ref{th:17_2} is applicable to prove convergence in the nonconvex case. We assume the functions $\f_1,\dots,\f_{K+1}$ in eq.~\eqref{e:nested} have continuous first derivatives w.r.t.\ both its input and its weights, so $E(\W,\Z)$ is differentiable w.r.t.\ \W\ and \Z.
\begin{thm}[Convergence of MAC/QP for nested problems]
  \label{th:MACQP}
  Consider the constrained problem~\eqref{e:mac} and its quadratic-penalty function $E_Q(\W,\Z;\mu)$ of~\eqref{e:mac-quadpen}. Given a positive increasing sequence $(\mu_k) \rightarrow \infty$, a nonnegative sequence $(\tau_k) \rightarrow 0$, and a starting point $(\W^0,\Z^0)$, suppose the QP method finds an approximate minimizer $(\W^k,\Z^k)$ of $E_Q(\W^k,\Z^k;\mu_k)$ that satisfies $\norm{\nabla_{\W,\Z}{E_Q(\W^k,\Z^k;\mu_k)}} \le \tau_k$ for $k=1,2,\dots$ Then, $\lim_{k\rightarrow\infty}{(\W^k,\Z^k)} = (\W^*,\Z^*)$, which is a KKT point for the problem~\eqref{e:mac}, and its Lagrange multiplier vector has elements $\blambda^*_n = \lim_{k\rightarrow\infty}{-\mu_k \, (\Z^k_n - \F(\Z^k_n,\W^k;\x_n))},\ n=1,\dots,N$.
\end{thm}
\begin{proof}
  It follows by applying theorem~\ref{th:17_2} to the constrained problem~\eqref{e:mac} and by noting that $\lim_{k\rightarrow\infty}{(\W^k,\Z^k)} = (\W^*,\Z^*)$ exists and that the constraint gradients are linearly independent. We prove these two statements in turn.

  The limit of the sequence $((\W^k,\Z^k))$ exists because the objective function $E(\W,\Z)$ of the MAC-constrained problem (hence the QP function $E_Q(\W,\Z;\mu)$) are lower bounded and have continuous derivatives. % This is loose, we probably need conditions on E_1 to have a minimiser, and on f1..f_{K+1} to be bounded.

  The constraint gradients are l.i.\ at any point $(\W,\Z)$ and thus, in particular, at the limit $(\W^*,\Z^*)$. To see this, let us first compute the constraint gradients. There is one constraint $C_{nkh}(\W,\Z) = z_{nkh} - f_{kh}(\z_{n,k-1};\W_k) = 0$ for each point $n=1,\dots,N$, layer $k=1,\dots,K$ and unit $h\in\calI(k)$, where we define $\calI(k)$ as the set of auxiliary coordinate indices for layer $k$ and $\z_{n0} = \x_n,\ n=1,\dots,N$. The gradient of this constraint is:
  \begin{align*}
    \frac{\partial C_{nkh}}{\partial\W_{k'}} &= - \delta_{kk'} \frac{\partial f_{kh}}{\partial\W_k},\ k=1,\dots,K \\
    \frac{\partial C_{nkh}}{\partial z_{n'k'h'}} &= \delta_{nn'} \left( \delta_{kk'} \delta_{hh'} - \delta_{k-1,k'} \frac{\partial f_{kh}}{\partial z_{n,k-1,h}} \right),\ n=1,\dots,N,\ k=1,\dots,K,\ h\in\calI(k).
  \end{align*}
  Now, we will show that these gradients are l.i.\ at any point $(\W,\Z)$. It suffices to look at the gradients w.r.t.\ \Z. Call $\alpha_{nkh} = \partial f_{kh}/\partial z_{n,k-1,h}$ for short. Constructing a linear combination of them and setting it to zero:
  \begin{equation*}
    \sum^N_{n=1}{\sum^K_{k=1}{\sum_{h\in\calI(k)}{ \lambda_{nkh} \frac{\partial C_{nkh}}{\partial\Z'} }}} = \0.
  \end{equation*}
  This implies, for the gradient element corresponding to $z_{n'k'h'}$:
  \begin{multline*}
    \sum^N_{n=1}{\sum^K_{k=1}{\sum_{h\in\calI(k)}{ \lambda_{nkh} \delta_{nn'} \left( \delta_{kk'} \delta_{hh'} - \delta_{k-1,k'} \alpha_{nkh} \right) }}} = \lambda_{n'k'h'} - \sum_{h\in\calI(k'+1)}{ \lambda_{n',k'+1,h} \alpha_{n',k'+1,h} } = 0 \\
    \Longrightarrow \lambda_{n'k'h'} = \sum_{h\in\calI(k'+1)}{ \lambda_{n',k'+1,h} \alpha_{n',k'+1,h} }.
  \end{multline*}
  Applying this for $k'=K,\dots,1$:
  \begin{itemize}
  \item For $k'=K$: $\lambda_{n'Kh'} = 0,\ n'=1,\dots,N,\ h'\in\calI(K)$.
  \item For $k'=K-1$: $\lambda_{n',K-1,h'} = \sum_{h\in\calI(K)}{ \lambda_{n',K,h} \alpha_{n',K,h} } = 0,\ n'=1,\dots,N,\ h'\in\calI(K-1)$.
  \item \dots
  \item For $k'=1$: $\lambda_{n',1,h'} = \sum_{h\in\calI(2)}{ \lambda_{n',2,h} \alpha_{n',2,h} } = 0,\ n'=1,\dots,N,\ h'\in\calI(1)$.
  \end{itemize}
  Hence, all the coefficients $\lambda_{nkh}$ are zero and the gradients are l.i.
\end{proof}
In practice, as with any continuous optimization problem, convergence may occur in pathological cases to a stationary point of the constrained problem rather than a minimizer.

In summary, MAC/QP defines a continuous path $(\W^*(\mu),\Z^*(\mu))$ which, under some mild assumptions (essentially, that we minimize $E_Q(\W,\Z;\mu)$ increasingly accurately as $\mu\rightarrow\infty$), converges to a stationary point (typically a minimizer) of the constrained problem~\eqref{e:mac}, and thus to a minimizer of the nested problem~\eqref{e:nested}.

We also have the following result (for simplicity, we give it for $K=1$ layer).
\begin{thm}
  \label{th:QP2}
  If a stationary point of the QP function for the problem of theorem~\ref{th:MAC-nested-equiv-KKT} satisfies $\z_n = \f_1(\x_n;\W_1)$, then it is also a stationary point of the nested problem for all $\mu \ge 0$.
\end{thm}
\begin{proof}
  The QP function is:
  \begin{equation*}
    E_Q(\W_1,\W_2,\Z;\mu) = \frac{1}{2} \sum^N_{n=1}{\norm{\y_n-\f_2(\z_n;\W_2)}^2} + \frac{\mu}{2} \sum^N_{n=1}{\norm{\z_n-\f_1(\x_n;\W_1)}^2}.
  \end{equation*}
  A stationary point of $E_Q$ must satisfy the equations:
  \begin{align}
    \label{e:QPstat-pt1}
    \frac{\partial E_Q}{\partial\W_1} &= -\mu \sum^N_{n=1}{\frac{\partial\f_1^T}{\partial\W_1}(\x_n) \, (\z_n-\f_1(\x_n))} = \0 \\
    \label{e:QPstat-pt2}
    \frac{\partial E_Q}{\partial\W_2} &= -\sum^N_{n=1}{\frac{\partial\f_2^T}{\partial\W_2}(\z_n) \, (\y_n-\f_2(\z_n)) } = \0 \\
    \label{e:QPstat-pt3}
    \frac{\partial E_Q}{\partial\z_n} &= \J_{\f_2}(\z_n)^T (\y_n-\f_2(\z_n)) + \mu (\z_n-\f_1(\x_n)) = \0,\ n=1,\dots,N.
  \end{align}
  If $\z_n = \f_1(\x_n;\W_1)$ for $n=1,\dots,N$, then from eq.~\eqref{e:QPstat-pt3} $\J_{\f_2}(\z_n)^T (\y_n-\f_2(\z_n)) = \0$ for $n=1,\dots,N$ and $(\W_1,\W_2)$ satisfies eqs.~\eqref{e:stat-pt1}--\eqref{e:stat-pt2}, so it is a stationary point of the nested problem for all $\mu \ge 0$.
\end{proof}
Remarks:
\begin{itemize}
\item Since the QP minimizer approaches the constraints from their infeasible side, the assumption $\z_n = \f_1(\x_n;\W_1)$ for $n=1,\dots,N$ does not hold unless $E_1(\W_1,\W_2) = 0$ there.
\item The converse of theorem~\ref{th:QP2} is not generally true: if $(\W_1,\W_2)$ is a stationary point of the nested problem, then defining $\z_n = \f_1(\x_n;\W_1)$ for $n=1,\dots,N$ we have that eqs.~\eqref{e:QPstat-pt1}--\eqref{e:QPstat-pt2} hold but eq.~\eqref{e:QPstat-pt3} does not.
\item Theorem~\ref{th:QP2} does not imply that the function $E_Q(\W,\Z;\mu)$ (for $\mu > \bar{\mu}$) is an exact penalty function for the objective $E(\W,\Z)$, for this we need the opposite: that any local solution of the MAC-constrained problem is a local minimizer of $E_Q(\W,\Z;\mu)$.
\end{itemize}

% \bibliographystyle{abbrvnat}
% \bibliography{macp,macp-xref}

\end{document}